\documentclass{article}

\usepackage{arxiv}
\usepackage[margin=1in]{geometry}
\usepackage{microtype}
\usepackage{graphicx}
\usepackage{subfigure}
\usepackage{booktabs} 
\usepackage{natbib}
\usepackage[affil-it]{authblk}

\usepackage{subfiles}
\usepackage[colorlinks,citecolor=blue,urlcolor=blue,linkcolor=blue,linktocpage=true]{hyperref}
\usepackage{libertine}
\usepackage{libertinust1math}
\usepackage[T1]{fontenc}

\usepackage[textsize=tiny,
 disable
]{todonotes}
\usepackage{soul}

\newcommand{\todoch}[2][]{\todo[color=blue!20!white]{Ch: #2}}

\newcommand{\todoc}[2][]{\todo[color=orange!20!white]{Cs: #2}}

\newcommand{\revise}[1]{#1}
\newcommand{\revisedelete}[1]{}

\newcommand{\logpi}{{\pi_{\mathrm{log}}}}

\usepackage{amsmath,amsfonts,bm}
\usepackage{amssymb,amsthm}
\usepackage{mathtools}
\usepackage{algorithm,algorithmic}
\usepackage{natbib}
\usepackage{xcolor}
\usepackage{hyperref}
\usepackage{url}
\usepackage{etoc}
\usepackage{cancel}
\usepackage[capitalize]{cleveref}
\usepackage{xfrac}

\crefname{proposition}{Proposition}{Propositions}
\crefname{theorem}{Theorem}{Theorems}
\crefname{lemma}{Lemma}{Lemmas}
\crefname{update_rule}{Update}{Updates}
\crefname{algorithm}{Algorithm}{Algorithms}

\def\eqref#1{equation~\ref{#1}}

\def\1{\bm{1}}

\renewcommand{\epsilon}{\varepsilon}

\DeclareMathAlphabet{\mathsfit}{\encodingdefault}{\sfdefault}{m}{sl}
\SetMathAlphabet{\mathsfit}{bold}{\encodingdefault}{\sfdefault}{bx}{n}

\def\cA{{\mathcal{A}}}
\def\cB{{\mathcal{B}}}

\def\cD{{\mathcal{D}}}
\def\cE{{\mathcal{E}}}
\def\cF{{\mathcal{F}}}

\def\cL{{\mathcal{L}}}
\def\cM{{\mathcal{M}}}
\def\cN{{\mathcal{N}}}

\def\cP{{\mathcal{P}}}

\def\cS{{\mathcal{S}}}

\def\cX{{\mathcal{X}}}

\def\sI{{\mathbb{I}}}

\def\sP{{\mathbb{P}}}

\newcommand{\KL}{D_{\mathrm{KL}}}

\DeclareMathOperator*{\argmin}{arg\,min}

\newtheorem{theorem}{Theorem}
\newtheorem{lemma}{Lemma}

\newtheorem{definition}{Definition}
\newtheorem{proposition}{Proposition}
\newtheorem{remark}{Remark}

\newtheorem{corollary}{Corollary}

\newcommand{\PP}{\mathbb{P}}
\newcommand{\RR}{\mathbb{R}}
\newcommand{\EE}{\mathbb{E}}

\newcommand{\norm}[1]{\|#1\|}

\newcommand{\addeq}{\addtocounter{equation}{1}\tag{\theequation}}
\makeatletter
\newcommand{\pushright}[1]{\ifmeasuring@#1\else\omit\hfill$\displaystyle#1$\fi\ignorespaces}
\newcommand{\pushleft}[1]{\ifmeasuring@#1\else\omit$\displaystyle#1$\hfill\fi\ignorespaces}
\makeatother

\newlength\tocrulewidth
\newcommand{\N}{\mathbb{N}}
\newcommand{\rS}{\mathrm{S}}
\newcommand{\rA}{\mathrm{A}}
\allowdisplaybreaks

\title{
The Curse of Working with Policy-induced Data in Batch Reinforcement Learning
}

\author{
{Chenjun Xiao}$^{1,2}$\thanks{{\color{red} Notes on revision:} we fix an error in previous upper bound results \cref{thm:plug-in-ub}. The original result states that the plug-in algorithm is nearly minimax optimal when using an uniform policy as the logging policy.  \cref{sec:revision-discussion} gives the details of the revision. } \ \
{Ilbin Lee}$^1$\ \ 
{Bo Dai}$^2$\ \  
{Dale Schuurmans}$^{1,2}$\ \ 
{Csaba Szepesvari}$^{1,3}$
}
\affil{
$^1$University of Alberta\quad
$^2$Google Research, Brain Team\quad
$^3$DeepMind \\
 {\{chenjun,ilbin,szepesva,daes\}@ualberta.ca},\ 
 {bodai@google.com}
}

\begin{document}

\maketitle

\begin{abstract}

In high stake applications, active experimentation may be considered too risky and thus data are often collected passively.
While in simple cases, such as in bandits, passive and active data collection are similarly effective, the price of passive sampling can be much higher when collecting data from a system with controlled states.
The main focus of the current paper is the characterization of this price.
For example, when learning in episodic finite state-action Markov decision processes (MDPs) with 
$\rS$ states and $\rA$ actions,
we show that even with the best (but passively chosen) logging policy,
$\Omega(\rA^{\min(\rS-1, H)}/\varepsilon^2)$ episodes are necessary (and sufficient) to obtain an $\epsilon$-optimal policy, where $H$ is the length of episodes.
Note that this shows that the sample complexity blows up exponentially compared to the case 
of active data collection, a result which is not unexpected, but, as far as we know, have not been published beforehand and perhaps the form of the exact expression is a little surprising.
%
We also extend these results in various directions, such as other criteria or learning in the presence of function approximation, with similar conclusions. 
A remarkable feature of our result is the sharp characterization of the exponent that appears, which is critical for understanding what makes passive learning hard.

\end{abstract}

\section{Introduction}
\vspace{-0.1in}

Batch reinforcement learning (RL) broadly refers to
the problem of finding a policy with high expected return in a stochastic control problem
when only a batch of data collected from the controlled system is available.
Here, we consider this problem for finite state-action Markov decision processes (MDPs),
with or without function approximation,
when the data is in the form of trajectories obtained by following some \emph{logging policy}.
In more details, the trajectories are composed of sequences of states, actions, and rewards, where the action is chosen by the logging policy, and the next states and rewards follow the distributions specified by the MDP's transition parameters.  

There are two subproblems underlying batch RL:
the \emph{design problem}, where the learner needs to specify 
a data collection mechanism that will be used to collect the batch data; 
and the \emph{policy optimization problem}, where the learner needs to specify the algorithm
that produces a policy given the batch data. 
For the design problem, often times one can use an adaptive data collection process where the next action to be taken is determined by the past data.
Another way to say this is that the data collection is done in an \emph{active} way. 
Recent theoretical advances in reward-free exploration \citep[e.g.,][]{
jin2020reward,kaufmann2021adaptive,zhang2021near} 
show that one can design algorithms to collect a batch data set with only polynomial samples to have good coverage over all possible scenarios in the environment.  
Near-optimal policies can be obtained for any given reward functions using standard policy optimization algorithms with the collected data. 
A complication arises in applications,
such as healthcare, education, autonomous driving, or hazard management,
where active data collection is either impractical or dangerous \citep{levine2020offline}. 
In these applications, the best one could do is to collect data using a fixed, logging policy, which needs to be chosen \emph{a priori}, that is before the data collection process begins, so that the stakeholders can approve it. Arguably, this is the most natural problem setting to consider in batch learning. 
The fundamental questions are:
how should one choose the logging policy so as to maximize the chance of obtaining a good policy with as little data as possible and
how many samples are sufficient and necessary to obtain a near optimal policy given a logging policy, and which algorithm to use to obtain such a policy?  

Perhaps surprisingly, before our paper, these questions remained unanswered.
In particular, while much work have studied the sample complexity of batch RL,
the results in these works are focusing only on the policy optimization subproblem and as such fall short in providing an answer to our questions.
In particular, some authors give sample complexity upper and lower bounds as 
a function of a parameter, $d_m$, which could be 
the smallest visit probability of state-action pairs under the logging policy
\citep{chen2019information,yin2020asymptotically,yin2021nearope,ren2021nearly,uehara2021finite,yin2021optimal,xie2021batch,xie2021bellman}, 
or the smallest ratio of visit probabilities of the logging versus the optimal policies, 
again over all state-action pairs
\citep{liu2019off,liu2020provably,yin2021nearopo,jin2021pessimism,rashidinejad2021bridging,xie2021policy}. 
The sample complexity results depend polynomially on $1/d_m$, $\rS,\rA$ and $H$, where $H$ is the episode length or the effective horizon.
Although these results are valuable in informing us the policy optimization step of batch RL,
they provide no clue as to how to choose the logging policy to get a high value for $d_m$
and whether $d_m$ will be uniformly bounded from below when adopting such a logging policy.
In particular, 
if we take the first definition for $d_m$, in some MDPs $d_m$ will be zero regardless of the logging policy if some state is not accessible from the initial distribution.
While this predicts an infinite sample complexity for our problem, this is clearly too conservative, since if a state is not accessible under any policy, it is unimportant to learn about it.
This is corrected by the second definition.
However,  even with this definition it remains unclear whether $d_m$ will be uniformly bounded away from zero for an MDP with a fixed number of states and actions and the best instance-agnostic choice of the logging policy. 
The lower bounds in these work also fail to provide a lower bound for our setting.
This is because in these lower bounds the instance will include an adversarially chosen logging policy, again falling short of helping us.
Essentially, these are the gaps that we fill with this paper.

In particular, 
we first show that in tabular MDPs the number of transitions necessary and sufficient to obtain a good policy, \emph{the sample complexity of learning}, is an exponential function of the minimum of the number of states and the planning horizon. 
In more details, we prove that the sample complexity of obtaining $\epsilon$-optimal policies is at least $\Omega(\rA^{\min(\rS-1, H+1)})$ for $\gamma$-discounted problems, where $\rS$ is the number of states, $\rA$ is the number of actions (per state), and $H$ is the effective horizon defined as 
$H=\lfloor \tfrac{\ln(\sfrac{1}{\epsilon})}{\ln(\sfrac1\gamma)} \rfloor$.
For finite horizon problems with horizon $H$, we prove the analogue
 $\Omega(\rA^{\min(\rS-1, H)}/\varepsilon^2)$ lower bound. 
These results for tabular MDPs immediately imply 
exponential lower bounds when linear value function approximation is applied with $\rS$ replaced by $d$, the number of features. 
We also show that warm starts (when one starts with a policy which is achieving almost as much as the optimal policy) do not help either, crushing the hope that one the ``curse'' of passivity can be broken by adopting a straightforward two-phase data collection process \citep{bai2019provably,zhang2020almost,gao2021provably}. 
We then establish nearly matching upper bounds for both
the plug-in algorithm and pessimistic algorithm, showing that the sample complexity
behaves essentially as shown in the lower bounds. 
While the upper bounds for these two algorithms may be off by a polynomial factor, we do not expect the pessimistic algorithm to have a major advantage over the plug-in method in the worst-case setting.
In fact, the recent work of \citet{xiao2021optimality} established this in a rigorous fashion for the bandit setting by showing an algorithm independent lower bound that matched the upper bound for both the plug-in method and the pessimistic algorithm. 
In the average reward case we show that the sample complexity is infinite.


How do our results relate to the expectations of RL practitioners (and theoreticians)?
We believe that most members of our community recognize that batch RL is hard at the fundamental level. 
Yet, it appears that many in the community are still highly optimistic about batch RL,
as demonstrated by the numerous empirical papers that document successes of various levels and kinds
\citep[e.g.,][]{laroche2019safe,kumar2019stabilizing,wu2019behavior,jaques2019way,agarwal2020optimistic,KiRaNeJo20,yu2020mopo,gulcehre2020rl,fu2020d4rl}, or by the optimistic tone of the above-mentioned theoretical results. 
The enthusiasm of the community is of course commendable and nothing is farthest from our intentions than to break it.
In connection to this, we would like to point out that our results show that if \emph{either $H$, or $\rS$} (or $d$ when we have $d$ features) is fixed, batch RL is in fact tractable.
Yet, the lower bound assures us that we cannot afford batch RL if \emph{both} of these parameters are large, a result which one should not hide from. 
Perhaps the most important finding here is the curious interplay between the horizon and the number of states (more generally, we expect a complexity parameter of the MDP to stand here), which is reasonable yet we find it non-obvious. Certainly, the proof that shows that the interplay is ``real'' required some new ideas.
Returning to the empirical works,  recent studies identify the tandem effect from the
issue of function approximation in the batch RL with passive data collection \citep{ostrovski2021difficulty}. 
Our results suggest that there is a need to rethink how batch RL methods are benchmarked.
In particular, a tedious examination of the benchmarks shows that the promising results are almost always produced on data sets that are collected by a noisy version of a near-optimal policy. The problem is that this choice biases the development of algorithms towards those that exploit this condition (e.g., the pessimistic algorithm), yet, this mode of data collection is unrealistic.%
\footnote{\citet{qin2021neorl} points to another problem with the benchmarks; namely that they fail to compare to the noise-free version of the near-optimal policy used in the data collection. 
\citet{brandfonbrener2021offline} observe that simply doing one step of constrained/regularized policy improvement using an
value estimate of the behavior policy performs surprisingly well in offline RL benchmarks.
They hypothesize that the strong performance is due to a combination of
favorable structure in the environment and behavior policy. }

Another highlight of our result is that it implies an exponential gap between the sample complexity of passive and active learning. 
Recently,
a similar conclusion is drawn in the paper of \citet{zanette2020exponential}, which served as the main motivation for our work. 
\citeauthor{zanette2020exponential} demonstrated an exponential separation for the case when batch learning is used in the presence of linear function approximation. A careful reading of the paper shows that the lower bound shown there does not apply to the tabular setting 
as the data collection method of this paper 
allows sampling from any distribution over the state-action space,
which, in the tabular setting is sufficient for polynomial sample complexity.

\section{Notation and Background} 

\paragraph{Notation} 
We let $\RR$ denote the set of real numbers, and for a positive integer $i$, let $[i] = \{0,\dots,i-1\}$ be the set of integers from $0$ to $i-1$.  We also let $\N=\{0,1,\dots\}$ be the set of nonnegative integers and $\N_+=\{1,2,\dots\}$ be the set of positive integers.
For a finite set $\cX$, we use $\Delta(\cX)$ to denote the set of probability distributions over $\cX$. 
We also use the same notation for infinite sets when the set has a clearly identifiable measurability structure such as $\RR$, which in this context is equipped by the $\sigma$-algebra of Borel measurable sets.
We use $\sI$ to denote the indicator function. We also use $\mathbf{1}$ to be the identically one function/vector; the domain/dimension is so that the expression that involves $\mathbf{1}$ is well-defined.

We consider finite {Markov decision processes (MDPs)} given by $M=( \cS, \cA, Q)$, where $\cS$ is a finite state space, $\cA$ is a finite action space, 
$Q$ is a stochastic kernel from the set $\cS \times \cA$ of state-action pairs to $\RR \times \cS$. 
In particular, for any $(s,a)\in \cS \times \cA$, $Q(\cdot|s,a)$ gives a distribution over pairs composed of a real number and a state. If $(R,S')\sim Q(\cdot|s,a)$, $R$ is interpreted as a random reward incurred and $S'$ the random next state when action $a$ is used in state $s$. Since the identity of the states and actions plays no role, without loss of generality, in what follows we assume that $\cS = [\rS]$ and $\cA = [\rA]$ for some $\mathrm{S},\mathrm{A}$ positive integers. 

Every MDP also induces a transition ``function'', $P: \cS \times \cA \to \Delta(\cS)$, which is the marginal of $Q$ with respect to $S'$, and an immediate mean reward function $r: \cS \times \cA \to \RR$, which is so that for any $(s,a)\in \cS \times \cA$, $r(s,a)$ gives the mean of $R$ above.
Since $\cS$ and $\cA$ are finite, without loss of generality, we assume that the immediate mean rewards lie in the $[-1,1]$ interval. We also assume that the reward distribution is $\rho$-subgaussian with a constant $\rho>0$. 
For simplicity, we assume that $\rho=1$.
For any (memoryless) policy $\pi:\cS\rightarrow \Delta(\cA)$, we define $P^\pi$ to be the transition matrix on state-action pairs induced by $\pi$, where 
$P^\pi_{(s,a), (s', a')} := \pi(a' | s') P(s' | s,a)$. 
We denote $\cM(\cS,\cA)$ the set of MDPs that satisfy the properties stated in this paragraph.
Since the identity of states and actions is unimportant, we also use $\cM(\rS,\rA)$ to denote the set of MDPs with $\rS$ states and $\rA$ actions (say, over the canonical sets $\cS = [\rS]$ and $\cA = [\rA]$). \todoc{I am still not (entirely) happy about the consistency of using $\cS,\cA$ and $\rS,\rA$.
Somehow, the identity of $\cS$ and $\cA$ are completely irrelevant, so they should somehow be deemphasized. I just don't know what's the best way of doing this.
}

We use $\EE^\pi$ to denote the expectation operator under the distribution $\mathbb{P}^\pi$ induced by the interconnection of policy $\pi$ and the MDP $M$ on trajectories $(S_0,A_0,R_0,S_1,A_1,R_1,\dots)$ formed of an infinite sequence of state, action, reward triplets.
Here, it is assumed that the initial state is randomly chosen from a fixed distribution; the dependence of $\sP^\pi$ on this distribution is suppressed. 
Further, under $\mathbb{P}^\pi$, the distribution of $S_{t+1}$ follows $P(\cdot|S_t,A_t)$ given the history $H_t = (S_0,A_0,R_0,\dots,R_{t-1},S_t,A_t)$ while the distribution of $A_t$ follows $\pi(\cdot|S_t)$ given $H_t$.
To minimize clutter the notation also suppresses the dependence on the MDP whenever it is clear which MDP is referred. \revise{A \emph{deterministic policy} $\pi$ is a policy such that for any state $s\in \cS$, $\pi(\cdot|s)$ puts all the probability mass on a single action.}

\revise{
Beyond memoryless policies, we will also allow for \emph{mixture} policies. A mixture policy $\pi$ is a collection of $k$ policies $\pi_1,\dots,\pi_k: \cS \to \Delta(\cA)$ and a distribution $p=(p_1,\dots,p_k)\in \Delta([k])$ over $[k]$ with $k>0$ a positive integer. The interconnection of $\pi = (p,\pi_1,\dots,\pi_k)$ and an MDP also induces a distribution over trajectories as above, with some differences. In particular, the idea is that at the beginning of the interaction $K$ is chosen, independently of $S_0$, so that $\mathbb{P}^\pi(K=i)=p_i$. Furthermore, 
the distribution of $A_t$ given $H_t$ and $K$ is $\pi_K(\cdot|S_t)$.
As a result, $\sP^{\pi} =\sum_i p_i \sP^{\pi_i}$.
}

For the \emph{discounted total reward criterion} with discount factor $0\leq \gamma <1$
the state value function $v^\pi:\cS\rightarrow \RR$ under $\pi$ is defined as,
\begin{align*}
v^\pi(s) := \EE^{\pi} \left[ \sum_{t=0}^\infty \gamma^t r(S_t, A_t)\, \Big\vert\, S_0 = s\right]\,.
\end{align*}
For $\mu\in\Delta(\cS)$ and $v: \cS \to \RR$,
we define $v(\mu) = \sum_{s\in \cS} \mu(s) v(s) (= \EE_{s\sim \mu}[v(s)])$.
The state-action value function of $\pi$, $q^\pi:\cS\times\cA\rightarrow \RR$, is defined as,
\begin{align*}
q^\pi(s,a) := r(s,a) + \gamma \sum\nolimits_{s'} P(s' | s,a) v^\pi(s')\, .
\end{align*}
The standard goal in a finite MDP under the discounted criterion is to identify the optimal policy $\pi^*$ that maximizes the value function in every state $s\in\cS$ such that
$
v^*(s) = \sup_{\pi} v^\pi(s)\, .
$
In this paper though, we consider the less demanding problem of finding a policy $\pi$ that maximizes $v^\pi(\mu)$ for a fixed initial state distribution $\mu$, i.e., finding a policy $\pi$ which achieves, or nearly achieves $v^*(\mu)$.

For an initial state distribution $\mu\in\Delta(\cS)$ and a policy $\pi$, 
we define the (unnormalized) discounted occupancy measure $\nu_\mu^\pi$ induced by $\mu$, $\pi$, and the MDP as 
\begin{align*}
\nu_\mu^\pi(s,a) := \sum_{t=0}^\infty \gamma^t \sP^\pi(S_t=s, A_t=a|S_0\sim \mu)\, .
\end{align*} 
The value of a policy can be represented as an inner product between the immediate reward function $r$ and the occupancy measure $\nu^\pi_\mu$
\begin{align*}
v^\pi(\mu) = \sum\nolimits_{s,a} r(s,a) \nu_\mu^\pi(s,a) =  \langle \nu_\mu^\pi, r \rangle\, .
\end{align*}
For $\varepsilon>0$, we define the effective horizon 
$
H_{\gamma, \varepsilon} :=   \lfloor \frac{\ln (1 / \epsilon)}{\ln (1 / \gamma)} \rfloor
$.%
\footnote{This is different with the normally used effective horizon, $\tfrac{\ln(\sfrac{1}{\epsilon(1-\gamma)})}{\ln(\sfrac{1}{\gamma})}$, with only a $\ln(1/(1-\gamma))$ factor.} 
In the case of \emph{fixed-horizon} policy optimization, instead of a discount factor, one is given a horizon $H>0$ and the value of a policy $\pi$ given $\mu$ is redefined to be $v^\pi(\mu) = \mathbb{E}^\pi[ \sum_{t=0}^{H-1} r(S_t,A_t) \big\vert \, S_0 \sim \mu ]$. As before, the goal is to identify a policy whose value is close to $v^*(\mu) = \sup_{\pi} v^{\pi}(\mu)$.
Finally, in the \emph{average reward} setting, $v^\pi(\mu)$ is redefined to be 
\begin{align*}
v^\pi(\mu)= \liminf_{T\to\infty} 
\EE^{\pi} \left[ \frac{1}{T}\sum_{t=0}^{T-1} r(S_t, A_t) \, \Big\vert\, S_0 \sim \mu \right]\,.
\end{align*}
For a given $\epsilon>0$, in all the various settings, we say that $\pi$ is \emph{$\epsilon$-optimal} in MDP $M$ given $\mu$ if $v^\pi(\mu)\ge v^*(\mu)-\epsilon$.

\section{Batch Policy Optimization}\label{sec:BPO} 

We consider \emph{policy optimization} in a \emph{batch mode}, or, in short, 
\emph{batch policy optimization} (BPO). 
A \emph{BPO problem} for a fixed sample size $n$ is given by the tuple
$\cB =(\cS,\cA,\mu,n,\cP)$ where $\cS$ and $\cA$ are finite sets, 
 $\mu$ is a probability distribution over $\cS$, $n$ is a positive integer, and
$\cP$ is a set of  MDP-distribution pairs of the form $(M,G)$, where 
$M\in \cM(\cS,\cA)$ is an MDP over $(\cS,\cA)$
and $G$ is a probability distribution over
$(\cS \times \cA \times \RR \times \cS)^n$. 
In what follows a pair $(M,G)$ of the above form will be called a BPO instance. 

A \emph{BPO algorithm} for a given sample size $n$ and sets $\cS,\cA$
takes data $\cD\in (\cS \times \cA \times \RR \times \cS)^n$ 
and returns a policy $\pi$ (possibly history-dependent).
Ignoring computational aspects, 
we will identify BPO algorithms with (possibly randomized) maps $\cL: (\cS \times \cA \times \RR \times \cS)^n \to \Pi$, 
where $\Pi$ is the set of all policies. 
The aim is to find BPO algorithms that find near-optimal policies with high probability on every instance within a BPO problem:
\begin{definition}[$(\varepsilon, \delta)$-sound algorithm]
\label{def:edsound}
Fix $\epsilon>0$ and $\delta\in (0,1)$.
A BPO algorithm $\cL$ is $(\varepsilon, \delta)$-sound
on instance $(M,G)$ given initial state distribution $\mu$
if 
\begin{align*}
\sP_{\cD \sim G}\left( v^{\cL(\cD)}(\mu) > v^*(\mu) - \varepsilon \right) > 1 - \delta\,,
\end{align*}
where the value functions are for the MDP $M$.
Further, we say that a BPO algorithm is $(\epsilon,\delta)$-sound on a BPO problem $\cB=(\cS,\cA,\mu,n,\cP)$ if it is sound on any $(M,G)\in \cP$ given the initial state distribution $\mu$.
\end{definition}



\paragraph{Data collection mechanisms}
A \emph{data collection mechanism} is a way of arriving at a distribution $G$ over the data given an MDP and some other inputs, such as the sample size.
We consider two types of data collection mechanisms.
One of them is governed by a distribution over the state-action pairs, the other is governed by a policy and a way of deciding how a fixed sample size $n$ should be split up into episodes in which the policy is followed.
We call the first $SA$-sampling, the second policy-induced data collection.
\begin{itemize}
\item \emph{$SA$-sampling:} 
An $SA$-sampling scheme is specified by
a probability distribution $\mu_{\text{log}}\in\Delta(\cS\times\cA)$ over the state-action pairs.
For a given sample size $n$, $\mu_{\log}$ together with an MDP $M$ induces a distribution 
$G_n(M,\mu_{\log})$ over $n$ tuples
$\cD = (S_i,A_i,R_i,S_i')_{i=0}^{n-1}$ so that the elements of this sequence form an i.i.d. sequence such that for any $i\in [n]$, $(S_i,A_i) \sim \mu_{\text{log}}$, $(R_i,S_{i}') \sim Q(\cdot|S_i,A_i)$.  

\item \emph{Policy-induced data collection}  
A policy induced data collection scheme is specified by  $(\logpi,\mathbf{h})$, where \revise{$\logpi$ is the
 \emph{logging policy}, which can be a mixture policy},
and $\mathbf{h} = (\mathbf{h}_n)_{n\ge 1}$: For each $n\ge 1$, $\mathbf{h}_n$ is an $m$-tuple $(h_j)_{j\in [m]}$ of positive integers for some $m$, specifying the length of the $m$ episodes in the data whose total length is $n$. Then, for any $n$, the pair $(\logpi,\mathbf{h}_n)$ together with an MDP $M$ and an initial distribution $\mu$ induces a distribution
$G(M,\logpi,\mathbf{h}_n,\mu)$ over the $n$ tuples $\cD=(S_i,A_i,R_i,S_i')_{i=0}^{n-1}$ as follows:
The data consists of $m$ episodes, with episode $j\in [m]$ having length $h_j$ and taking the form
$\tau_j=( S_0^{(j)},  A_0^{(j)},  R_0^{(j)}, \dots,  S_{h_j-1}^{(j)},  A_{h_j-1}^{(j)},  R_{h_j-1}^{(j)},  S_{h_j}^{(j)})$, where $ S_0^{(j)}\sim\mu$, $ A_t^{(j)}\sim \logpi(\cdot| S_{t}^{(j)})$, $(R_t^{(j)}, S_{t+1}^{(j)}) \sim Q(\cdot |  S_t^{(j)},  A_t^{(j)})$. 
Then, for $i\in [n]$, $(S_i,A_i,R_i,S_i') = ( S_{t}^{(j)}, A_{t}^{(j)},R_{t}^{(j)},S_{t+1}^{(j)})$ where 
$j\in [m]$,  $t\in [h_j]$ are unique integers such that $i=\sum_{j'<j} h_j + t$. \todoc{we are nice here because the initial distribution matches the initial distribution that the optimization happens for. perhaps we can add somewhere a comment about what happens when this is not the case. I guess we just need to pay the probability ratio..}
We call $\mathbf{h}$ \emph{data splitting scheme}. 
\end{itemize}

Now, under $SA$-sampling, 
the sets $\cS$, $\cA$, a logging distribution $\mu_{\log}$ and state-distribution $\mu$ over the respective sets give rise to the BPO problem $\cB(\mu_{\log},\mu,n)=(\cS,\cA,\mu,n,\cP(\mu_{\log},n))$, 
where $\cP(\mu_{\log},n)$ is the set of all pairs of the form 
$(M,G_n(M,\mu_{\log}))$, where $M\in \cM(\cS,\cA)$ is an MDP with the specified state-action spaces
and $G_n(M,\mu_{\log})$ is defined as above. 
Similarly, 
a fixed policy $\pi_{\log}$, fixed episode lengths $\mathbf{h}\in \mathbb{N}_+^m$ for some $m$ integer and a fixed state-distribution $\mu$ give rise to a BPO problem $\cB(\pi_{\log},\mu,\mathbf{h})=(\cS,\cA,\mu,|\mathbf{h}|, \cP(\pi_{\log},\mathbf{h}))$, where $\cP(\pi_{\log},\mathbf{h})$ 
is the set of pairs of the form 
$(M,G(M,\pi_{\log},\mathbf{h},\mu))$ where $M\in \cM(\cS,\cA)$ and $G(M,\pi_{\log},\mathbf{h},\mu)$ is a distribution as defined above. Here, we use $|\mathbf{h}|$ to denote $\sum_{s=0}^{m-1} h_s$ which is the sample size specified by $\mathbf{h}$.

The \emph{sample-complexity of BPO with $SA$-sampling} for a given pair $(\varepsilon,\delta)$ 
and a criterion (discounted, finite horizon, or average reward)
is the smallest integer $n$ such that for each $\mu$ there exists 
a logging distribution $\mu_{\log}$ and 
a BPO algorithm $\cL$ for this sample size such that $\cL$ is $(\varepsilon,\delta)$-sound on
the BPO problem $\cB(\mu_{\log},\mu,n)$. 
Similarly, \emph{the sample-complexity of BPO with policy-induced data collection} 
for a given pair $(\varepsilon,\delta)$
and a criterion 
 is the smallest integer $n$ such that for each $\mu$ there exists 
a logging policy $\pi_{\log}$ and episode lengths $\mathbf{h}\in \N_+^m$ with $|\mathbf{h}|=n$
and a BPO algorithm that is $(\varepsilon,\delta)$-sound on $\cB(\logpi,\mu,\mathbf{h})$.

The $SA$-sampling based data collection is realistic when there is a simulator that allows this type of data collection  \citep{agarwal2020model,azar2013minimax,cui2020plug,li2020breaking}. Besides this scenario,
it is hard to imagine a case when $SA$-sampling can be realistically applied.
Indeed, in most practical settings, data collection happens by following some policy, usually from the same initial state distribution that is used in the objective of policy optimization. 

For policy-induced data collection, 
a key restriction on the logging policy is that it is chosen without any knowledge of the MDP. Moreover, that the logging policy is memoryless rules out any adaptation to the MDP.
The intention here is to model a ``tabula rasa'' setting, which is relevant when one must find a good policy in a completely new environment but only passive data collection is available.
However, our lower bound shows 
that there is not much to be gained even if the logging policy is known to be a good policy: If the goal is to improve the suboptimality level of the logging policy, by saying, a factor of two, the exponential sample complexity lower bound still applies.

From a statistical perspective, the main difference between these two data collection mechanisms is that for policy-induced data-collection the distribution of $(S_i,A_i)$ will depend on the specific MDP instance,
while this is not the case for $SA$-sampling. As we shall see, this makes BPO under $SA$-sampling provably exponentially more efficient.

\section{Lower Bounds}
\label{sec:lb}

We first give a lower bound on the sample complexity for BPO when the data available for learning is obtained by following some logging policy: 

\begin{theorem}[Exponential sample complexity with policy-induced data collection in discounted problems]
\label{thm:pi-induced-lb}
For any positive integers $\rS$ and $\rA$, discount factor $\gamma\in [0,1)$ and a pair $(\epsilon,\delta)$ such that $0<\epsilon<1/2$ and $\delta\in (0,1)$, 
any $(\epsilon,\delta)$-sound algorithm needs at least $\Omega( \rA^{\min(\rS-1,H_{\gamma, 2\varepsilon}+1)} \ln (1/\delta) )$ \emph{episodes} of any length  
with policy-induced data collection for MDPs with $\rS$ states and $\rA$ actions 
under the $\gamma$-discounted total expected reward criterion. 
The result remains true if the MDPs are restricted to have deterministic transitions.
\end{theorem}
\begin{remark}
\label{rem:random-reward}
Random rewards are not essential in proving \cref{thm:pi-induced-lb} as long as stochastic transitions are allowed: First, the proof can be modified to use Bernoulli rewards and stochastic transitions can be used to emulate Bernoulli rewards.
Note also that for $\rho$-subgaussian random reward, the sample complexity in \cref{thm:pi-induced-lb} becomes $\Omega(\max\{1, \rho^2\} \rA^{\min(\rS-1,H+1)} \ln (1/\delta))$.
The maximum appears exactly because stochastic transitions can emulate Bernoulli rewards.
\end{remark}

Simplifying things a bit, the theorem states that the sample complexity is exponential as the number of states and the planning horizon grow together and without a limit.
Note that this is in striking contrast to sample complexity of learning actively, or even with $SA$-sampling, as we shall soon discuss it.  
The hard MDP instance used to construct the lower bound is adopted from the combination lock problem \citep{whitehead1991complexity}.  
The detailed proof is provided in the supplementary material, as are the proofs of the other statements.
By realizing that tabular MDPs can be considered as using one-hot features, the exponential lower bound still hold for linear function approximation. 

\begin{corollary}[Exponential sample complexity with linear function approximation in discounted problem]
\label{cor:linear-exponential-lb} 
Let $d,\rA$ be positive integers. Then
the same result as Theorem~\ref{thm:pi-induced-lb} with $S$ replaced by $d$ holds
when the data collection happens for some MDP $M\in \cM(\cS,[\rA])$ 
and in addition to the dataset the learner is also given access to a featuremap 
$\phi:\cS\rightarrow \RR^d$ such that for every policy $\pi$ of this MDP,
there exists $\theta\in\RR^d$ such that $v^\pi(s)=\phi(s)^\top \theta$, $\forall s\in\cS$, where $v^\pi$ is the value function of $\pi$ in $M$.
%
%
The result also remains true if the MDPs are restricted to have deterministic transitions. 
\end{corollary}

One may wonder about whether this exponential complexity can be avoided if more is assumed about the logging policy. In particular, one may hope that improving on an already good logging policy (i.e., one that is close to optimal) should be easier.
Our next result shows that this is not the case. 

\begin{corollary}[Warm starts do not help]
\label{cor:warmstartdoesnothelp}
Fix $\pi_{\log}$, $0<\varepsilon<1/2$, $\delta\in (0,1)$, $\rS$ and $\rA$ as before.
Let $\cM_{2\varepsilon}^{\logpi}$ denote a set of MDPs with deterministic transitions, state space $\cS = [\rS]$ and action space $\cA = [\rA]$ such that $\logpi$ is $2\varepsilon$-optimal for all $M\in\cM_{2\varepsilon}^{\logpi}$. Then for any length of sampled episodes, any $(\varepsilon, \delta)$-sound algorithm needs at least $\Omega( \rA^{\min(\rS-1,H+1)} \ln (1/\delta))$ episodes, where $H=H_{\gamma, 2\varepsilon}$. 
\end{corollary}

For fixed-horizon policy optimization, we have the following result  similar to \cref{thm:pi-induced-lb}. 
\begin{theorem}[Exponential sample complexity with policy-induced data collection in finite-horizon problems]
\label{thm:pi-induced-lb-fin}
For any positive integers $\rS$ and $\rA$, planning horizon $H>0$ and a pair $(\epsilon,\delta)$ such that $0<\epsilon<1/2$ and $\delta\in (0,1)$, 
any $(\epsilon,\delta)$-sound algorithm needs at least $\Omega(\rA^{\min(\rS-1,H)} \ln (1/\delta)/\varepsilon^2 )$ episodes 
with policy-induced data collection for MDPs with $\rS$ states and $\rA$ actions 
under the $H$-horizon total expected reward criterion. 
The result also remains true if the MDPs are restricted to have deterministic transitions.
\end{theorem}
\cref{rem:random-reward} and \cref{cor:linear-exponential-lb,cor:warmstartdoesnothelp} 
 also remain essentially true; we omit these to save space. 
As shown in the next result, the sample complexity could be even worse in average reward MDPs. 
The different sample complexities of the average reward problems and the two previous settings can be explained as follows. In discounted and finite-horizon problems, rewards beyond the planning horizon do not have to be observed in data to find a near optimal policy. In contrast, this is not the case for the average reward criterion: 
rewards at states that are ``hard'' to reach may have to be observed enough in data. 
Thus, the fact that the planning horizon is finite is crucial for a finite sample complexity.

\begin{theorem}[Infinite sample complexity with policy-induced data collection in average reward problems] 
For any positive integers $\rS$ and $\rA$, any pair $(\epsilon,\delta)$ such that $0<\epsilon<1/2$ and $\delta\in (0,1)$, the sample complexity of BPO with policy-induced data collection for MDPs with $\rS$ states and $\rA$ actions under the average reward criterion is infinite. 
\label{thm:pi-induced-lb-avg}
\end{theorem}

For $SA$-sampling, the sample complexity becomes polynomial in the relevant quantities:
Staying with discounted problems,
this is implied by the results of \citet{agarwal2020model} who study plug-in methods when a generative model is used to generate the same number of observations for each state-action pair. In particular, they show that in this setting, if $N$ samples are available in each state-action pair then the plug-in algorithm will find a policy with $v^\pi \ge v^* - \varepsilon \mathbf{1}$ 
provided that $N=\Omega(\ln \frac{\rS \rA}{(1-\gamma)\delta}/(\varepsilon^2 (1-\gamma)^3))$. This implies a sample complexity upper bound of size $\tilde O( \rS \rA H^3 \ln(1/\delta) /\varepsilon^2)$ where  $H=1/(1-\gamma)$, though for the stronger requirement that $\pi$ is optimal not only from $\mu$, but from each state. 
The upper bound is essentially matched by a lower bound 
 by \citet{SiWa18} who prove their result in Section D of their paper
 using a reduction to a result of \citet{azar2013minimax} that stated a similar sample complexity lower bound for estimating the optimal value. Our result is stronger than these results, which require the algorithm to produce a ``globally good'' policy, i.e., a policy that is near-optimal no matter the initial state, while in our result, the policy needs to be good only at a \emph{fixed} initial state distribution.
\begin{theorem}
\label{thm:pi-free-lb}
Fix any $\gamma_0>0$.
Then, there exist some constants $c_0,c_1>0$ such that for any
$\gamma\in [\gamma_0,1)$,
any positive integers $\rS$ and $\rA$, 
$\delta\in(0,1)$, and $0<\varepsilon\le c_0/(1-\gamma)$,
 the sample size $n$ needed by
any $(\varepsilon,\delta)$-sound algorithm that produces as output a memoryless policy and works with $SA$-sampling for MDPs with $\rS$ states and $\rA$ actions under the $\gamma$-discounted expected reward criterion must be so that 
is at least $c_1 \frac{\rS \rA \ln (1 / (4\delta))}{\varepsilon^2 (1-\gamma)^3}$.
\end{theorem}
%
%
Our proof for the lower bound essentially follows the ideas of \citet{azar2013minimax}, but an effort was made to make the proof more streamlined. In particular, the new proof uses Le Cam's method.
We leave it for future work to extend the result to algorithms whose output is not restricted to memoryless policies.
%



\section{Upper Bounds}
\label{sec:ub}

We now consider the ``plug-in'' algorithm for BPO and the discounted total expected reward criterion and will present a result for it that shows that this simple approach essentially matches the sample complexity lower bound of 
\cref{thm:pi-induced-lb}. 
For simplicity, we assume that the reward function is known.%
\footnote{
As noted also, e.g., by \citet{agarwal2020model}, the sample size requirements stemming 
from the need to obtain a sufficiently accurate estimate of the reward function is a lower order term compared to that needed to accurately estimate the transition probabilities.}
Given a batch of data, the plug-in algorithm 
uses sample means to construct estimates for the transition probabilities. These can then be fed into any MDP solver to get a policy.
The plug-in method is an obvious first choice that has proved its value in a number of different settings
\citep{agarwal2020model,azar2013minimax,cui2020plug,li2020breaking,ren2021nearly,xiao2021optimality}. 

For the details, let $\cD=(S_i, A_i, R_i, S'_i)_{i=0}^{n-1}$ be the data available to the algorithm.
We let
\begin{align*}
N(s, a, s') = \sum_{i=0}^{n-1} \sI\left( S_i=s, A_i=a, S_i'=s' \right)\, 
\end{align*}
denote the number of transitions observed in the data from $s$ to $s'$ while action $a$ is used.
We also let $N(s,a) = \sum_{s'}N(s, a, s')$ be the number of times the pair $(s,a)$ is seen in the data. Provided that the visit count $N(s,a)$ is positive, we let
\begin{align*}
\hat{P}(s' | s,a) =  \frac{N(s, a, s')}{N(s, a)} 
\end{align*}
be the estimated probability of transitioning to $s'$ given that $a$ is taken in state $s$.
We let $\hat{P}(s' |s,a) = 0$ for all $s'\in\cS$ when $N(s,a)$ is zero.%
\footnote{We note that the particular values chosen here do not have an essential effect on the results.
For example, when $\hat{P}(\cdot | s,a)$ is the uniform distribution over $\cS$, it will only effect the constant factor in \cref{thm:plug-in-ub} (see \cref{eq:nzero} in the appendix).} 
The plug-in algorithm returns a policy by solving the planning problem defined with $(\hat{P},r)$, exploiting that planning algorithms need only the mean rewards and the transition probabilities
\citep{puterman2014markov}.
By slightly abusing the definitions, we will hence treat $(\hat{P},r)$ as an MDP and denote it by $\hat{M}$. In the result stated below we also allow a little slack for the planner; i.e., the planner is allowed to return a policy which is $\epsilon_{\text{opt}}$-optimal.
The main result for this section is as follows:
\begin{theorem}
\label{thm:plug-in-ub}
Fix $\rS$, $\rA$, an MDP $M\in \cM(\rS,\rA)$ and a distribution $\mu$ on the state space of $M$.
Suppose $\delta>0$, $\varepsilon>0$, and $\varepsilon_{\text{opt}}>0$. 
Assume that the data is collected by following \revise{the uniform mixture of all deterministic policies} and it consists of $m$ episodes, each of length $H=H_{\gamma, (1-\gamma)\varepsilon/(2\gamma)}$.
Let $\hat{\pi}$ be any deterministic, $\varepsilon_{\text{opt}}$-optimal policy for $\hat{M}=(\hat{P},r)$ where $\hat{P}$ is the sample-mean based empirical estimate of the transition probabilities based on the data collected.
Then if 
\begin{align*}
m= \tilde{\Omega} \left(  \frac{\rS^3 \rA^{\min(H, \rS) + 2}  \ln\frac{1}\delta }{(1-\gamma)^{4}\epsilon^2} \right),
\end{align*}
where $\tilde\Omega$ hides log factors of $\rS,\rA$ and $H$,
we have 
$v^{\hat{\pi}}(\mu) \ge v^*(\mu) -4\varepsilon - \varepsilon_{\text{opt}}$
with probability at least $1-\delta$. 
\end{theorem}
\begin{remark}
Our proof technique for the upper bound can be directly applied to the fixed $H$-horizon setting and gives an identical result.%
\todoc{note that I added that the result does not change.}
\end{remark}
\revise{
\begin{remark}
Note that the leading terms of \cref{thm:pi-induced-lb} and \cref{thm:plug-in-ub}, as a function of $\rA,\rS$ and $H$ is
$\rA^{\min(H, \rS)}$ and they essentially match.
Thus, these results together imply that in this respect the
uniform mixture of deterministic policies gives an \emph{optimal design} for batch RL with policy-induced data. 
An intriguing question, which we leave open, is whether the memoryless policy that assigns the uniform distribution (the ``uniform policy'') has the same property. In particular, it remains to be shown whether the uniform policy enjoys a bound similar to that shown for the uniform mixture policy in \cref{thm:plug-in-ub}.
Similarly, while it can be shown that using a memoryless policy that is non-uniform leads to a larger lower bound than shown in \cref{thm:pi-induced-lb}, it remains unclear whether the same result holds for mixture policies. 
\end{remark}
}
We note that for BPO with policy-induced data collection, it is not possible to directly apply a reductionist approach based on analysis for $SA$-sampling, which requires a uniform lower bound on the number of transitions observed at all the state-action pairs, which could even be infinite. 
The key to avoid this is to show that the ratio of visit probabilities for an arbitrary policy vs. the uniform mixture of deterministic policies in step $t$ is at most $\rA^{\min(t+1,\rS)}$, as shown by \cref{prop:ratbd2} in the Appendix. 
We provide the proof of \cref{thm:plug-in-ub}, as well as an analogous result for the pessimistic policy \citep{jin2021pessimism,BuGeBe20,KiRaNeJo20,yu2020mopo,kumar2020conservative,liu2020provably,xiao2021optimality} 
in the Appendix.


\section{Related Work}
\label{sec:related}

Our work is motivated by that of \citet{zanette2020exponential} who considers 
the sample complexity of BPO in MDPs and linear function approximation.
One of the main results in this paper (Theorem 2) is that 
the $(1/2,1/2)$ sample complexity 
with a ``reinforced'' policy-induced data collection in MDPs whose optimal action-value function is realizable with a $d$-dimensional feature map given to the learner is at least $\Omega( (1/(1-\gamma))^{\frac{d-1}{2}} )$.
The ``reinforced'' data collection gives the learner full access to the transition kernel and rewards at states that are reachable from the start states with the policy (or policies) chosen.
Thus, the learner here has more information than in our setting, but the problem is made hard by the presence of linear function approximators.
As noted by \citeauthor{zanette2020exponential}, the same setting is trivially easy in the finite horizon setting, thus the result shows a separation between the infinite and finite horizon settings.
The weakness of this separation is that the ``reinforced data collection'' mechanism is unrealistic.
A second result in the paper (Theorem 3) shows that in the presence of function approximation, 
even under $SA$-sampling,
the sample complexity is still exponential in $d$ (as in Theorem 2 mentioned above) 
even when the features are so that the action-value functions of \emph{any} policy can be realized.
This exponential sample complexity is to be contrasted with the fully polynomial result available for the same setting when a generative model is available \citep{LaSzeWe20}. 
Thus, this result shows an exponential separation between passive and active learning.
It is interesting to note that this separation disappears in the tabular setting under $SA$-sampling.

For linear function approximation under $SA$-sampling a number of authors show related exponential (or infinite) sample complexity when the sampling distribution is chosen in a semi-adversarial way
\citep{amortila2020variant,WaFoKa:ICLR21,chen2021infinitehorizon} in the sense that it can be chosen to be the worst distribution among those 
that provide \emph{good coverage in the feature space} (expressed 
as a condition on the minimum eigenvalue of the feature second moment matrix).
The main message of these results is that good coverage in the feature space is insufficient for sample-efficient BPO. 
Since the hard examples in these works are tabular MDPs with $O(d)$ state-action pairs, the uniform distribution over the state-action space is sufficient to guarantee polynomial sample complexity in the same ``hard MDPs''. Hence, these hardness results also have a distinctly different feature than the hardness result we present.

A different line of research can be traced back to the work of \citet{LiMuSz15} 
who were concerned with statistically efficient
batch policy evaluation (BPE) with policy-induced data-collection.
The significance of this work for our paper is that at the end of the paper
the authors of this work added a sidenote 
stating that the \emph{sample complexity of finite horizon BPE must be exponential in the horizon}. 
Their example is a combination-lock type MDP, which served as an inspiration for the constructions we use in our lower bound proofs. 
No arguments are made for the suitability of the lower-bound for BPO, 
nor is a formal proof given for the exponential sample complexity for BPE. 
As such, our work can be seen 
as the careful examination of this remark in this paper and its adoption to BPO.
A closely related, but weaker observation, is that the (vanilla) importance sampling estimators for BPE suffer an exponential blow-up of the variance \citep{guo2017using},
a phenomenon that \citet{Liu18:Curse} call the \emph{curse of horizon}  in BPE.
This exponential dependence is also pointed out by \citet{jiang2016doubly}, who provide a lower bound on the asymptotic variance of any unbiased estimator in BPE.

Lately, much effort has been devoted to ``breaking'' this aforementioned curse. 
The basis of these works is the observation that 
 if sufficient coverage for the state-action space is provided by the logging policy,
the curse should not apply (i.e., plug-in estimators should work well).
Considering finite-horizon problems for now,
the coverage condition is usually expressed as a lower bound $d_m$ on the smallest visit probabilities
of the logging policy.
The main result here, due to \citet{yin2021nearope},
 is that the sample-complexity (or, better, episode-complexity) of BPO,
 with an inhomogeneous $H$-step transition structure and
 up to constant and logarithmic factors, is $H^3/(d_m \epsilon^2)$, achieved by the plug-in estimator.
According to  \citet{yin2021nearopo}, this complexity continues to hold for the discounted setting when it represents the ``step complexity'' (as opposed to ``episode complexity'').
The same work also removes a factor of $H$ both from the lower and upper bounds for the finite horizon setting with homogeneous transitions.
A further strengthening of the results for the homogeneous setting is due to \citet{ren2021nearly} who remove an additional $H$ factor under the assumption that the total reward in every episode belongs to the $[0,1]$ interval. 
These results justify the use of coverage as a way of describing the inherent hardness of BPO. 
These results are complementary to ours. 
The lower bound in these works for fixed $d_m$ is achieved by keeping the number of state-action pairs free, while we consider sample complexity for a fixed number of state-action pairs.

An alternative approach to characterize the sample-complexity of BPO is followed by 
\citet{jin2021pessimism} who, for the inhomogeneous transition kernel, finite-horizon setting,
consider a \emph{weighted error metric}. 
While their primary interest is in obtaining results for linear function approximation, their result can be simplified back to the tabular setting. 
Their main result then shows that 
the minimax expected value of this weighted metric is lower bounded by a universal constant, 
while the pessimistic algorithm can match this bound with polynomial factors. 
Note that the results that are phrased with the help of the minimum coverage probability can also be rewritten as results on the minimax error for a weighted error, where the weights would include the minimum coverage probabilities. All these results are complementary to each other.

Average reward  BPO with a parametric policy class for finite MDPs using policy-induced data is considered by \citet{LZS20}. The authors derive an ``efficient'' value estimator, and the policy returned is defined as the one that achieves the largest estimated value.
An upper bound on the suboptimality of the policy returned is given in terms of a number of quantities that relate to the policy parameterization provided that a coverage condition is satisfied similar to the coverage assumption discussed above.
\todoc{Can we specialize this to the case when the policy parameterization is trivial? They do something with Boltzmann policies, but even that is unnecessarily complicated.}
\todoc{We could add some references to the literature on best policy identification, though this problem has a very different flavor, which we should also note if we decide to add references about this literature.}

Finally, we note that there is extensive literature on BPE; the reader is referred to the works of
\citep{yin2021nearope,yin2020asymptotically,ren2021nearly,uehara2021finite,pananjady2020instance} 
and the references therein.
The most relevant works for $SA$-sampling are concerned with the sample complexity of planning with  generative models; see, e.g., 
\citep{azar2013minimax,agarwal2020model,yin2020asymptotically} and the references therein.

\vspace{-0.1in}
\section{Conclusion}
The main motivation for our paper is to fill a substantial gap in the literature of batch policy optimization:
While the most natural setting for batch policy optimization is when the data is obtained by following some policy, the sample complexity, the minimum number of observations necessary and sufficient to find a good policy, of batch policy optimization with data obtained this way has never been formally studied. 
Our results characterize how hard BPO under passive data collection exactly is and how the difficulty scales as the problem parameter changes. 
While our main result that, with a finite planning horizon, the sample complexity scales exponentially  is perhaps somewhat expected, this has never been formally established and should therefore be a valuable contribution to the field.
In fact, both the lower and the upper bound required considerable work to be rigorously establish and that the sample complexity is finite is less obvious in light of the previous results that involved ``minimum coverage'' as a superficial argument with these results suggest that the sample complexity could grow without bound if some state-action pairs have arbitrary small visit probabilities.
That these results, as far as the details are concerned, are non-obvious is also shown
by the gap that we could not close between the upper and lower sample complexity bounds.
Another non-obvious insight 
of our work is that warm starts provably cannot help in reducing the sample complexity.
Our results should be given even more significance by the fact that the tabular setting provides the foundation for most of the insights that lead to better algorithms in RL. 

\section{Acknowledgments}

This work is done when Chenjun Xiao was intern at Google Brain. 
Chenjun Xiao and Bo Dai would like to thank Ofir Nachum for providing feedback on a draft of this manuscript. 
Ilbin Lee is supported by Discovery Grant from NSERC.
Csaba Szepesv\'ari and Dale Schuurmans gratefully acknowledge funding from the Canada CIFAR AI Chairs Program, Amii and NSERC.

\bibliography{ref}

\begin{thebibliography}{57}
\providecommand{\natexlab}[1]{#1}
\providecommand{\url}[1]{\texttt{#1}}
\expandafter\ifx\csname urlstyle\endcsname\relax
  \providecommand{\doi}[1]{doi: #1}\else
  \providecommand{\doi}{doi: \begingroup \urlstyle{rm}\Url}\fi

\bibitem[Agarwal et~al.(2020{\natexlab{a}})Agarwal, Kakade, and
  Yang]{agarwal2020model}
Alekh Agarwal, Sham Kakade, and Lin~F Yang.
\newblock Model-based reinforcement learning with a generative model is minimax
  optimal.
\newblock In \emph{COLT}, pages 67--83, 2020{\natexlab{a}}.

\bibitem[Agarwal et~al.(2020{\natexlab{b}})Agarwal, Schuurmans, and
  Norouzi]{agarwal2020optimistic}
Rishabh Agarwal, Dale Schuurmans, and Mohammad Norouzi.
\newblock An optimistic perspective on offline reinforcement learning.
\newblock In \emph{International Conference on Machine Learning}, pages
  104--114. PMLR, 2020{\natexlab{b}}.

\bibitem[Amortila et~al.(2020)Amortila, Jiang, and Xie]{amortila2020variant}
Philip Amortila, Nan Jiang, and Tengyang Xie.
\newblock A variant of the {W}ang-{F}oster-{K}akade lower bound for the
  discounted setting.
\newblock \emph{arXiv preprint 2011.01075}, 2020.

\bibitem[Azar et~al.(2013)Azar, Munos, and Kappen]{azar2013minimax}
Mohammad~Gheshlaghi Azar, R{\'e}mi Munos, and Hilbert~J Kappen.
\newblock Minimax {PAC} bounds on the sample complexity of reinforcement
  learning with a generative model.
\newblock \emph{Machine learning}, 91\penalty0 (3):\penalty0 325--349, 2013.

\bibitem[Bai et~al.(2019)Bai, Xie, Jiang, and Wang]{bai2019provably}
Yu~Bai, Tengyang Xie, Nan Jiang, and Yu-Xiang Wang.
\newblock Provably efficient q-learning with low switching cost.
\newblock In \emph{Advances in Neural Information Processing Systems},
  volume~32, 2019.

\bibitem[Brandfonbrener et~al.(2021)Brandfonbrener, Whitney, Ranganath, and
  Bruna]{brandfonbrener2021offline}
David Brandfonbrener, William~F Whitney, Rajesh Ranganath, and Joan Bruna.
\newblock Offline rl without off-policy evaluation.
\newblock \emph{arXiv preprint arXiv:2106.08909}, 2021.

\bibitem[Buckman et~al.(2021)Buckman, Gelada, and Bellemare]{BuGeBe20}
Jacob Buckman, Carles Gelada, and Marc~G. Bellemare.
\newblock The importance of pessimism in fixed-dataset policy optimization.
\newblock In \emph{ICLR}, 2021.

\bibitem[Chen and Jiang(2019)]{chen2019information}
Jinglin Chen and Nan Jiang.
\newblock Information-theoretic considerations in batch reinforcement learning.
\newblock In \emph{International Conference on Machine Learning}, pages
  1042--1051. PMLR, 2019.

\bibitem[Chen et~al.(2021)Chen, Scherrer, and
  Bartlett]{chen2021infinitehorizon}
Lin Chen, Bruno Scherrer, and Peter~L. Bartlett.
\newblock Infinite-horizon offline reinforcement learning with linear function
  approximation: Curse of dimensionality and algorithm.
\newblock \emph{arXiv preprint 2103.09847}, 2021.

\bibitem[Cui and Yang(2020)]{cui2020plug}
Qiwen Cui and Lin Yang.
\newblock Is plug-in solver sample-efficient for feature-based reinforcement
  learning?
\newblock In \emph{NeurIPS}, volume~33, pages 6015--6026, 2020.

\bibitem[Fu et~al.(2020)Fu, Kumar, Nachum, Tucker, and Levine]{fu2020d4rl}
Justin Fu, Aviral Kumar, Ofir Nachum, George Tucker, and Sergey Levine.
\newblock D4rl: Datasets for deep data-driven reinforcement learning.
\newblock \emph{arXiv preprint arXiv:2004.07219}, 2020.

\bibitem[Gao et~al.(2021)Gao, Xie, Du, and Yang]{gao2021provably}
Minbo Gao, Tianle Xie, Simon~S Du, and Lin~F Yang.
\newblock A provably efficient algorithm for linear markov decision process
  with low switching cost.
\newblock \emph{arXiv preprint arXiv:2101.00494}, 2021.

\bibitem[Gulcehre et~al.(2020)Gulcehre, Wang, Novikov, Paine, Colmenarejo,
  Zolna, Agarwal, Merel, Mankowitz, Paduraru, et~al.]{gulcehre2020rl}
Caglar Gulcehre, Ziyu Wang, Alexander Novikov, Tom~Le Paine, Sergio~Gomez
  Colmenarejo, Konrad Zolna, Rishabh Agarwal, Josh Merel, Daniel Mankowitz,
  Cosmin Paduraru, et~al.
\newblock Rl unplugged: A suite of benchmarks for offline reinforcement
  learning.
\newblock \emph{arXiv preprint arXiv:2006.13888}, 2020.

\bibitem[Guo et~al.(2017)Guo, Thomas, and Brunskill]{guo2017using}
Zhaohan~Daniel Guo, Philip~S Thomas, and Emma Brunskill.
\newblock Using options and covariance testing for long horizon off-policy
  policy evaluation.
\newblock In \emph{NeurIPS}, pages 2489--2498, 2017.

\bibitem[Jaques et~al.(2019)Jaques, Ghandeharioun, Shen, Ferguson, Lapedriza,
  Jones, Gu, and Picard]{jaques2019way}
Natasha Jaques, Asma Ghandeharioun, Judy~Hanwen Shen, Craig Ferguson, Agata
  Lapedriza, Noah Jones, Shixiang Gu, and Rosalind Picard.
\newblock Way off-policy batch deep reinforcement learning of implicit human
  preferences in dialog.
\newblock \emph{arXiv preprint arXiv:1907.00456}, 2019.

\bibitem[Jiang and Li(2016)]{jiang2016doubly}
Nan Jiang and Lihong Li.
\newblock Doubly robust off-policy value evaluation for reinforcement learning.
\newblock In \emph{ICML}, pages 652--661, 2016.

\bibitem[Jin et~al.(2020)Jin, Krishnamurthy, Simchowitz, and Yu]{jin2020reward}
Chi Jin, Akshay Krishnamurthy, Max Simchowitz, and Tiancheng Yu.
\newblock Reward-free exploration for reinforcement learning.
\newblock In \emph{International Conference on Machine Learning}, pages
  4870--4879. PMLR, 2020.

\bibitem[Jin et~al.(2021)Jin, Yang, and Wang]{jin2021pessimism}
Ying Jin, Zhuoran Yang, and Zhaoran Wang.
\newblock Is pessimism provably efficient for offline {RL}?
\newblock In \emph{ICML}, 2021.

\bibitem[Kaufmann et~al.(2021)Kaufmann, M{\'e}nard, Domingues, Jonsson,
  Leurent, and Valko]{kaufmann2021adaptive}
Emilie Kaufmann, Pierre M{\'e}nard, Omar~Darwiche Domingues, Anders Jonsson,
  Edouard Leurent, and Michal Valko.
\newblock Adaptive reward-free exploration.
\newblock In \emph{Algorithmic Learning Theory}, pages 865--891, 2021.

\bibitem[Kidambi et~al.(2020)Kidambi, Rajeswaran, Netrapalli, and
  Joachims]{KiRaNeJo20}
Rahul Kidambi, Aravind Rajeswaran, Praneeth Netrapalli, and Thorsten Joachims.
\newblock {MOReL}: Model-based offline reinforcement learning.
\newblock In \emph{NeurIPS}, volume~33, pages 21810--21823, 2020.

\bibitem[Kumar et~al.(2019)Kumar, Fu, Tucker, and Levine]{kumar2019stabilizing}
Aviral Kumar, Justin Fu, George Tucker, and Sergey Levine.
\newblock Stabilizing off-policy q-learning via bootstrapping error reduction.
\newblock \emph{arXiv preprint arXiv:1906.00949}, 2019.

\bibitem[Kumar et~al.(2020)Kumar, Zhou, Tucker, and
  Levine]{kumar2020conservative}
Aviral Kumar, Aurick Zhou, George Tucker, and Sergey Levine.
\newblock Conservative q-learning for offline reinforcement learning.
\newblock In \emph{NeurIPS}, volume~33, pages 1179--1191, 2020.

\bibitem[Laroche et~al.(2019)Laroche, Trichelair, and
  Des~Combes]{laroche2019safe}
Romain Laroche, Paul Trichelair, and Remi~Tachet Des~Combes.
\newblock Safe policy improvement with baseline bootstrapping.
\newblock In \emph{International Conference on Machine Learning}, pages
  3652--3661. PMLR, 2019.

\bibitem[Lattimore and Szepesv{\'a}ri(2020)]{lattimore2020bandit}
Tor Lattimore and Csaba Szepesv{\'a}ri.
\newblock \emph{Bandit algorithms}.
\newblock Cambridge University Press, 2020.

\bibitem[Lattimore et~al.(2020)Lattimore, Szepesv\'ari, and Weisz]{LaSzeWe20}
Tor Lattimore, {Cs}aba Szepesv\'ari, and Gell\'ert Weisz.
\newblock Learning with good feature representations in bandits and in {RL}
  with a generative model.
\newblock In \emph{ICML}, pages 5662--5670, 2020.

\bibitem[Levine et~al.(2020)Levine, Kumar, Tucker, and Fu]{levine2020offline}
Sergey Levine, Aviral Kumar, George Tucker, and Justin Fu.
\newblock Offline reinforcement learning: Tutorial, review, and perspectives on
  open problems.
\newblock \emph{arXiv preprint arXiv:2005.01643}, 2020.

\bibitem[Li et~al.(2020)Li, Wei, Chi, Gu, and Chen]{li2020breaking}
Gen Li, Yuting Wei, Yuejie Chi, Yuantao Gu, and Yuxin Chen.
\newblock Breaking the sample size barrier in model-based reinforcement
  learning with a generative model.
\newblock \emph{NeurIPS}, 33:\penalty0 12861--12872, 2020.

\bibitem[Li et~al.(2015)Li, Munos, and Szepesv{\'a}ri]{LiMuSz15}
L.~Li, R.~Munos, and {Cs}. Szepesv{\'a}ri.
\newblock Toward minimax off-policy value estimation.
\newblock In \emph{AISTATS}, pages 608--616, 2015.

\bibitem[Liao et~al.(2020)Liao, Qi, and Murphy]{LZS20}
Peng Liao, Zhengling Qi, and Susan Murphy.
\newblock Batch policy learning in average reward {M}arkov decision processes.
\newblock \emph{arXiv 2007.11771}, 2020.

\bibitem[Liu et~al.(2018)Liu, Li, Tang, and Zhou]{Liu18:Curse}
Qiang Liu, Lihong Li, Ziyang Tang, and Dengyong Zhou.
\newblock Breaking the curse of horizon: Infinite-horizon off-policy
  estimation.
\newblock In \emph{NeurIPS}, pages 5361--5371, 2018.

\bibitem[Liu et~al.(2019)Liu, Swaminathan, Agarwal, and Brunskill]{liu2019off}
Yao Liu, Adith Swaminathan, Alekh Agarwal, and Emma Brunskill.
\newblock Off-policy policy gradient with state distribution correction.
\newblock \emph{arXiv preprint arXiv:1904.08473}, 2019.

\bibitem[Liu et~al.(2020)Liu, Swaminathan, Agarwal, and
  Brunskill]{liu2020provably}
Yao Liu, Adith Swaminathan, Alekh Agarwal, and Emma Brunskill.
\newblock Provably good batch off-policy reinforcement learning without great
  exploration.
\newblock In \emph{NeurIPS}, 2020.

\bibitem[Mitzenmacher and Upfal(2005)]{MiUp05:book}
Michael Mitzenmacher and Eli Upfal.
\newblock \emph{Probability and Computing: Randomized Algorithms and
  Probabilistic Analysis}.
\newblock Cambridge University Press, 2005.

\bibitem[Ostrovski et~al.(2021)Ostrovski, Castro, and
  Dabney]{ostrovski2021difficulty}
Georg Ostrovski, Pablo~Samuel Castro, and Will Dabney.
\newblock The difficulty of passive learning in deep reinforcement learning.
\newblock \emph{Advances in Neural Information Processing Systems}, 34, 2021.

\bibitem[Pananjady and Wainwright(2020)]{pananjady2020instance}
Ashwin Pananjady and Martin~J Wainwright.
\newblock Instance-dependent bounds for policy evaluation in tabular
  reinforcement learning.
\newblock \emph{IEEE Transactions on Information Theory}, 67\penalty0
  (1):\penalty0 566--585, 2020.

\bibitem[Puterman(2014)]{puterman2014markov}
Martin~L Puterman.
\newblock \emph{Markov decision processes: discrete stochastic dynamic
  programming}.
\newblock John Wiley \& Sons, 2014.

\bibitem[Qin et~al.(2021)Qin, Gao, Zhang, Xu, Huang, Li, Zhang, and
  Yu]{qin2021neorl}
Rongjun Qin, Songyi Gao, Xingyuan Zhang, Zhen Xu, Shengkai Huang, Zewen Li,
  Weinan Zhang, and Yang Yu.
\newblock Neorl: A near real-world benchmark for offline reinforcement
  learning.
\newblock \emph{arXiv preprint arXiv:2102.00714}, 2021.

\bibitem[Rashidinejad et~al.(2021)Rashidinejad, Zhu, Ma, Jiao, and
  Russell]{rashidinejad2021bridging}
Paria Rashidinejad, Banghua Zhu, Cong Ma, Jiantao Jiao, and Stuart Russell.
\newblock Bridging offline reinforcement learning and imitation learning: A
  tale of pessimism.
\newblock \emph{arXiv preprint arXiv:2103.12021}, 2021.

\bibitem[Ren et~al.(2021)Ren, Li, Dai, Du, and Sanghavi]{ren2021nearly}
Tongzheng Ren, Jialian Li, Bo~Dai, Simon~S Du, and Sujay Sanghavi.
\newblock Nearly horizon-free offline reinforcement learning.
\newblock \emph{arXiv preprint arXiv:2103.14077}, 2021.

\bibitem[Sidford et~al.(2018)Sidford, Wang, Wu, Yang, and Ye]{SiWa18}
Aaron Sidford, Mengdi Wang, Xian Wu, Lin~F. Yang, and Yinyu Ye.
\newblock Near-optimal time and sample complexities for solving {M}arkov
  decision processes with a generative model.
\newblock In \emph{NeurIPS}, pages 5192--5202, 2018.

\bibitem[Uehara et~al.(2021)Uehara, Imaizumi, Jiang, Kallus, Sun, and
  Xie]{uehara2021finite}
Masatoshi Uehara, Masaaki Imaizumi, Nan Jiang, Nathan Kallus, Wen Sun, and
  Tengyang Xie.
\newblock Finite sample analysis of minimax offline reinforcement learning:
  Completeness, fast rates and first-order efficiency.
\newblock \emph{arXiv preprint arXiv:2102.02981}, 2021.

\bibitem[Wang et~al.(2021)Wang, Foster, and Kakade]{WaFoKa:ICLR21}
Ruosong Wang, Dean~P. Foster, and Sham~M. Kakade.
\newblock What are the statistical limits of offline {RL} with linear function
  approximation?
\newblock In \emph{ICLR}, 2021.

\bibitem[Whitehead(1991)]{whitehead1991complexity}
Steven~D Whitehead.
\newblock Complexity and cooperation in q-learning.
\newblock In \emph{International Conference on Machine Learning}, pages
  363--367, 1991.

\bibitem[Wu et~al.(2019)Wu, Tucker, and Nachum]{wu2019behavior}
Yifan Wu, George Tucker, and Ofir Nachum.
\newblock Behavior regularized offline reinforcement learning.
\newblock \emph{arXiv preprint arXiv:1911.11361}, 2019.

\bibitem[Xiao et~al.(2021)Xiao, Wu, Lattimore, Dai, Mei, Li, Szepesv\'ari, and
  Schuurmans]{xiao2021optimality}
Chenjun Xiao, Yifan Wu, Tor Lattimore, Bo~Dai, Jincheng Mei, Lihong Li, Csaba
  Szepesv\'ari, and Dale Schuurmans.
\newblock On the optimality of batch policy optimization algorithms.
\newblock In \emph{ICML}, 2021.

\bibitem[Xie and Jiang(2021)]{xie2021batch}
Tengyang Xie and Nan Jiang.
\newblock Batch value-function approximation with only realizability.
\newblock In \emph{International Conference on Machine Learning}, pages
  11404--11413. PMLR, 2021.

\bibitem[Xie et~al.(2021{\natexlab{a}})Xie, Cheng, Jiang, Mineiro, and
  Agarwal]{xie2021bellman}
Tengyang Xie, Ching-An Cheng, Nan Jiang, Paul Mineiro, and Alekh Agarwal.
\newblock Bellman-consistent pessimism for offline reinforcement learning.
\newblock \emph{arXiv preprint arXiv:2106.06926}, 2021{\natexlab{a}}.

\bibitem[Xie et~al.(2021{\natexlab{b}})Xie, Jiang, Wang, Xiong, and
  Bai]{xie2021policy}
Tengyang Xie, Nan Jiang, Huan Wang, Caiming Xiong, and Yu~Bai.
\newblock Policy finetuning: Bridging sample-efficient offline and online
  reinforcement learning.
\newblock \emph{arXiv preprint arXiv:2106.04895}, 2021{\natexlab{b}}.

\bibitem[Yin and Wang(2020)]{yin2020asymptotically}
Ming Yin and Yu-Xiang Wang.
\newblock Asymptotically efficient off-policy evaluation for tabular
  reinforcement learning.
\newblock In \emph{AISTATS}, pages 3948--3958, 2020.

\bibitem[Yin and Wang(2021)]{yin2021optimal}
Ming Yin and Yu-Xiang Wang.
\newblock Optimal uniform ope and model-based offline reinforcement learning in
  time-homogeneous, reward-free and task-agnostic settings.
\newblock \emph{arXiv preprint arXiv:2105.06029}, 2021.

\bibitem[Yin et~al.(2021{\natexlab{a}})Yin, Bai, and Wang]{yin2021nearope}
Ming Yin, Yu~Bai, and Yu-Xiang Wang.
\newblock Near-optimal provable uniform convergence in offline policy
  evaluation for reinforcement learning.
\newblock In \emph{AISTATS}, pages 1567--1575, 2021{\natexlab{a}}.

\bibitem[Yin et~al.(2021{\natexlab{b}})Yin, Bai, and Wang]{yin2021nearopo}
Ming Yin, Yu~Bai, and Yu-Xiang Wang.
\newblock Near-optimal offline reinforcement learning via double variance
  reduction.
\newblock \emph{arXiv preprint arXiv:2102.01748}, 2021{\natexlab{b}}.

\bibitem[Yu et~al.(2020)Yu, Thomas, Yu, Ermon, Zou, Levine, Finn, and
  Ma]{yu2020mopo}
Tianhe Yu, Garrett Thomas, Lantao Yu, Stefano Ermon, James~Y. Zou, Sergey
  Levine, Chelsea Finn, and Tengyu Ma.
\newblock {MOPO:} model-based offline policy optimization.
\newblock In \emph{NeurIPS}, pages 14129--14142, 2020.

\bibitem[Yu et~al.(2021)Yu, Kumar, Rafailov, Rajeswaran, Levine, and
  Finn]{Yu2021combo}
Tianhe Yu, Aviral Kumar, Rafael Rafailov, Aravind Rajeswaran, Sergey Levine,
  and Chelsea Finn.
\newblock {COMBO:} conservative offline model-based policy optimization.
\newblock \emph{arXiv preprint arXiv:2102.08363}, 2021.

\bibitem[Zanette(2021)]{zanette2020exponential}
Andrea Zanette.
\newblock Exponential lower bounds for batch reinforcement learning: Batch {RL}
  can be exponentially harder than online {RL}.
\newblock In \emph{ICML}, 2021.

\bibitem[Zhang et~al.(2020)Zhang, Zhou, and Ji]{zhang2020almost}
Zihan Zhang, Yuan Zhou, and Xiangyang Ji.
\newblock Almost optimal model-free reinforcement learningvia
  reference-advantage decomposition.
\newblock In \emph{Advances in Neural Information Processing Systems},
  volume~33, 2020.

\bibitem[Zhang et~al.(2021)Zhang, Du, and Ji]{zhang2021near}
Zihan Zhang, Simon Du, and Xiangyang Ji.
\newblock Near optimal reward-free reinforcement learning.
\newblock In \emph{International Conference on Machine Learning}, pages
  12402--12412, 2021.

\end{thebibliography}
\bibliographystyle{plainnat}

\newpage

\newpage
\appendix

\onecolumn
\section*{Appendix: The Curse of Passive Data Collection in Batch Reinforcement Learning}
The purpose of this appendix is to give the proofs for the results in the main paper.

\revise{
\section{Discussion on Revision: Using Mixture of Stationary Policies as the Logging Policy }
\label{sec:revision-discussion}

In this revision, we fix an error in previous upper bound results \cref{thm:plug-in-ub}: the original result states that the plug-in algorithm is nearly minimax optimal when using an uniform policy as the logging policy. 
The problem with this result was that it made the claim we now make for the uniform mixture policy \cref{prop:ratbd2} for the uniform policy and this claim was made based on an incorrect argument. We still do not know whether the bound that holds for the uniform mixture policy can also be shown for the uniform policy. Our current understanding of the relevant quantity under the uniform policy is stated in \cref{prop:ratbd}. Thus, the upper bound is now stated for the uniform mixture policy. As a result, to make the theory complete, we also modified the lower bound and their proofs in \cref{sec:lb-proofs} to allow mixture policies (the earlier statements were restricted to memoryless policies). We leave the question whether using a fixed stationary policy is also sufficient for achieving minimax optimality as an open problem.
}

\section{An absolute bound on the state-action probability ratios under the uniform logging policy and the uniform mix of deterministic policies}
\label{sec:absratiobound}
\newcommand{\trg}{\mathrm{trg}}
\newcommand{\Det}{\mathrm{DET}}
For a policy $\pi$, $t\ge 0$, $(s,a)\in \cS \times \cA$ let
\begin{align*}
\nu_{\mu,t}^\pi(s,a) := \sP^\pi(S_t=s, A_t=a|S_0\sim \mu)\, .
\end{align*} 
As noted beforehand, ratios of these marginal probabilities appear in previous upper (and lower) bounds on how well the value of a target policy $\pi_{\trg}$ can be estimated given data from a logging policy $\pi_{\log}$. To minimize clutter, let $\nu_{\mu,t}^{\trg}$ stand for $\nu_{\mu,t}^{\pi_{\trg}}$ and, similarly,
let $\nu_{\mu,t}^{\log}$ stand for $\nu_{\mu,t}^{\pi_{\log}}$.
The purpose of this section is to present 
a short calculation that bounds $\frac{\nu_{\mu,t}^{\trg}(s,a)}{\nu_{\mu,t}^{\log}(s,a)}$, which is the ratio that appears in the previously mentioned bounds. First, we bound this ratio for the uniform logging policy when $\pi_{\log}(a|s) = 1/\rA$. 
\begin{proposition}\label{prop:ratbd}
When $\pi_{\log}$ is the uniform policy, for any $t\ge 0$, $(s,a)\in \cS \times \cA$ and  $\pi_{\trg}$ is any target policy, 
\begin{align}
\frac
{\nu_{\mu,t}^{\trg}(s,a)}
{\nu_{\mu,t}^{\log}(s,a)}
\le \rA^{t+1}\,.
\label{eq:ratbd}
\end{align}
Furthermore, there exists a MDP and $(s,a)\in\cS\times\cA$ such that $\nu_{\mu,t}^{\trg}(s,a) / \nu_{\mu,t}^{\log}(s,a) \geq \rA^{\rS}$ for $t+1\geq \rS$. 
\end{proposition}
\begin{proof}
Fix an arbitrary pair $(s_t,a_t)\in \cS \times \cA$. 
Let $s_{0:t}$ denote a sequence $(s_0,\dots,s_t)$ of states and let $a_{0:t}$ denote a sequence $(a_0,\dots,a_t)$ of actions.
We have
\begin{align*}
\nu_{\mu,t}^{\trg}(s_t,a_t)
& =
\sum_{
\substack{
s_{0:t-1}\\
a_{0:t-1}}
} 
\mu(s_0)\pi_{\trg}(a_0|s_0)p(s_1|s_0,a_0) \dots \pi_{\trg}(a_{t-1}|s_{t-1}) p(s_t|s_{t-1},a_{t-1}) \pi_{\trg}(a_t|s_t) \\
& \le
\sum_{
\substack{
s_{0:t-1}\\
a_{0:t-1}}
} 
\rA^{t+1} \mu(s_0)\pi_{\log}(a_0|s_0)p(s_1|s_0,a_0) \dots \pi_{\log}(a_{t-1}|s_{t-1}) p(s_t|s_{t-1},a_{t-1}) \pi_{\log}(a_t|s_t) \\
& = \rA^{t+1} \nu_{\mu,t}^{\log}(s_t,a_t)\,.
\end{align*}
Dividing both sides by $\nu_{\mu,t}^{\log}(s_t,a_t)$ gives the desired bound. The inequality is tight when there is only one possible path $(s_0,a_0,s_1,a_1,\ldots,s_t,a_t)$ to $(s_t,a_t)$ in an MDP and the target policy is the deterministic policy taking the actions in the unique path.  

We now show an example to prove the second part of the claim. 
Consider a MDP with two states $\cS = \{s_1, s_2\}$ and two actions $\cA = \{a_1, a_2\}$. 
The MDP always starts at state $s_1$ at the beginning of an episode, that is $\mu(s_1)=1, \mu(s_2)=0$. 
At state $s_1$ under action $a_2$, the MDP transits to $s_2$, while it stays at $s_1$ under $a_1$. 
State $s_2$ is absorbing under any action. 
Let $\pi_\trg$ be a deterministic policy $\pi_\trg\in\Det$ such that $\pi_\trg(s_1)=a_1$, it can be verified that $\nu_{\mu, t}^{\pi_\trg}(s_1, a_1) = 1$ for any $t\geq 1$. 
First consider the uniform logging policy.  
For $t\geq 1$,
\begin{align*}
\nu^{\log}_{\mu, t} (s_1, a_1) = \nu^{\log}_{\mu, t-1} (s_1, a_1)  \sP(s_1, a_1 | s_1, a_1) 
=
\frac{1}{2} \nu^{\log}_{\mu, t-1} (s_1, a_1)  
=
\dots
=
\frac{1}{2^{t+1}}\, .
\end{align*}
Thus for $t+1\geq \rS$, 
$\nu_{\mu, t}^{\pi_\trg}(s_1, a_1) / \nu_{\mu, t}^{\log}(s_1, a_1) = 2^{t+1} \geq \rA^{\rS}$. 

Now consider using the mixture of deterministic policies as $\pi_{\log}$.  
One can see that $\nu_{\mu, t}^{\log}(s_1, a_1) = \frac{1}{2}$ since it will always stay at $s_1$ as long as a policy always selects $a_1$ at $s_1$, 
and there are 2 out of 4 such deterministic policies.

\end{proof}
From the proof it is clear that the result continues to hold even if the target policy depends on the full history.

The uniform mix of all deterministic policies selects one among all $\rA^\rS$ deterministic policies at the beginning where each of them can be chosen with probability $1/\rA^\rS$. Then, it uses the chosen deterministic policy in all time periods. A key difference between the two logging policies we consider is that the uniform policy randomly chooses an action every time the system reaches a state whereas the uniform mix of all deterministic policies randomly selects a deterministic policy at the beginning and then uses the deterministic policy afterwards, thus, a random selection happens only once. Now we prove a counterpart of Proposition~\ref{prop:ratbd} for the uniform mix of all deterministic policies. 

\begin{proposition}\label{prop:ratbd2}
When $\pi_{\log}$ is the uniform mixture of deterministic policies, for any $t\ge 0$, $(s,a)\in \cS \times \cA$ and  $\pi_{\trg}$ is any deterministic target policy, 
\begin{align}
\frac
{\nu_{\mu,t}^{\trg}(s,a)}
{\nu_{\mu,t}^{\log}(s,a)}
\le \rA^{\min(t+1,S)}\,.
\label{eq:ratbd}
\end{align}
\end{proposition}
\begin{proof}
Let $\Det$ be the set of stationary deterministic policies over $\cS$ and $\cA$. Fix an arbitrary pair $(s_t,a_t)\in \cS \times \cA$. We have
\begin{align*}
\nu_{\mu,t}^{\log}(s_t,a_t)&=\frac{1}{A^S}\sum_{\pi\in \Det} \nu_{\mu,t}^{\pi}(s_t,a_t)\\
 &=\frac{1}{A^S}\sum_{\pi\in \Det}\sum_{\substack{s_{0:t-1}\\a_{0:t-1}}} \mu(s_0)p(s_1|s_0,a_0) \dots p(s_t|s_{t-1},a_{t-1})\mathbf{1}(a_0=\pi(s_0), \ldots, a_t=\pi(s_t))\\
 &=\frac{1}{A^S}\sum_{\substack{s_{0:t-1}\\a_{0:t-1}}}\sum_{\pi\in \Det}\mu(s_0)p(s_1|s_0,a_0) \dots p(s_t|s_{t-1},a_{t-1})\mathbf{1}(a_0=\pi(s_0), \ldots, a_t=\pi(s_t))\\
 &=\frac{1}{A^S}\sum_{\substack{s_{0:t-1}\\a_{0:t-1}}}\mu(s_0)p(s_1|s_0,a_0) \dots p(s_t|s_{t-1},a_{t-1})\sum_{\pi\in \Det}\mathbf{1}(a_0=\pi(s_0), \ldots, a_t=\pi(s_t))\\
 &\ge \frac{1}{A^S}\sum_{\substack{s_{0:t-1}\\a_{0:t-1}}}\mu(s_0)p(s_1|s_0,a_0) \dots p(s_t|s_{t-1},a_{t-1})A^{S-t-1}\\
 &\ge \frac{1}{A^{t+1}}\nu_{\mu,t}^{\text{trg}}(s_t,a_t).
\end{align*}
Also, for $(s,a)\in \cS \times \cA$ we have
\begin{align*}
\nu_{\mu,t}^{\trg}(s,a) \le \sum_{\pi \in \Det} \nu_{\mu,t}^\pi(s,a) =\rA^{\rS} \frac{1}{\rA^{\rS}} \sum_{\pi \in \Det} \nu_{\mu,t}^\pi(s,a) = \rA^{\rS} \nu_{\mu,t}^{\log}(s,a),
\end{align*}
and this completes the proof. The inequality is tight when there is only one possible path $(s_0,a_0,s_1,a_1,\ldots,s_t,a_t)$ to $(s_t,a_t)$ in an MDP, the target policy is the deterministic policy taking the actions in the unique path, and the path does not repeat any state.  
\end{proof}


%

\section{Lower Bound Proofs}
\label{sec:lb-proofs}
Before these proofs, an equivalent form of $(\epsilon,\delta)$-soundness will be useful to consider.
Recall that $\cL$ is $(\epsilon,\delta)$-sound on instance $(M,G)$ if
\begin{align*}
\sP_{\cD \sim G}\left( v^{\cL(\cD)}(\mu) > v^*(\mu) - \varepsilon \right) > 1 - \delta\,,
\end{align*}
Now, $\sP_{\cD \sim G}\left( v^{\cL(\cD)}(\mu) > v^*(\mu) - \varepsilon \right) = 
1-\sP_{\cD \sim G}\left( v^{\cL(\cD)}(\mu) \le v^*(\mu) - \varepsilon \right)$.
Hence, $\cL$ is $(\epsilon,\delta)$-sound on instance $(M,G)$ if and only if
\begin{align*}
\sP_{\cD \sim G}\left( v^{\cL(\cD)}(\mu) \le v^*(\mu) - \varepsilon \right)<\delta\,.
\end{align*}
Finally, by reordering, this last display is equivalent to
\begin{align*}
\sP_{\cD \sim G}\left(v^*(\mu) -  v^{\cL(\cD)}(\mu) \ge \varepsilon  \right)<\delta\,.
\end{align*}
Thus, $\cL$ is \emph{not} $(\epsilon,\delta)$ sound on $(M,G)$ if 
\begin{align}
\sP_{\cD \sim G}\left(v^*(\mu) -  v^{\cL(\cD)}(\mu) \ge \varepsilon \right)\ge \delta\,.
\label{eq:noted}
\end{align}

\begin{figure}[t]
\begin{center}
\includegraphics[width=12cm]{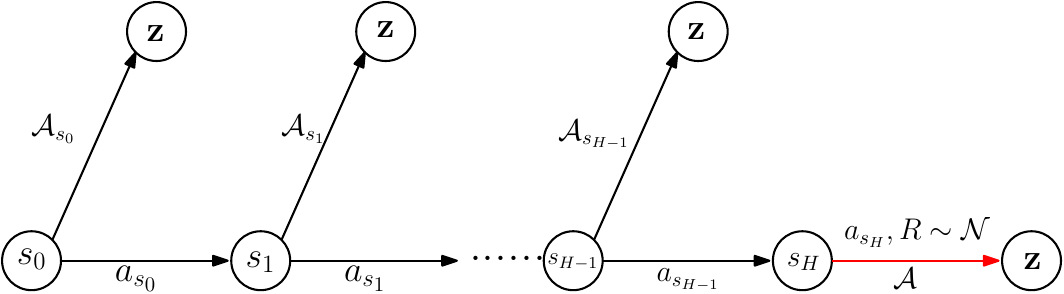}
\end{center}
\caption{
Illustration of the MDPs used in the proof of \cref{thm:pi-induced-lb}. 
For $\varepsilon>0$, Let $H = H_{\gamma, 2\varepsilon}$. 
The state space consists of two parts $\cS=\{s_0,s_1,\dots,s_H\} \cup \{z\}$, where $s_0$ is the initial state, $z$ is a self-absorbing state. 
For any $s\in\cS$, let $a_s = \argmin_{a} \logpi(a|s)$ be the action with minimal chance of being selected by $\logpi$, and $\cA_{s} = \cA \setminus \{a_s\}$. 
The transitions and rewards are as follows: State $z$ is absorbing under any action.
For $i\in\{0, \dots, H-1\}$, at state $s_i$ under action $a_{s_i}$
the MDP transits to $s_{i+1}$, while it transits to $z$ under any other actions. From $s_H$, the next state is also $z$ under any action.
The rewards are deterministically zero for any state-action pair except when the state is $s_H$, and action $a_{S_H}$ is taken, when it is random with either a positive or negative mean.
}  
\label{fig:lb}
\end{figure}

We will need some basic concepts, definitions, and results from information theory. For two probability measures, $P$ and $Q$ over a common measurable space, we use $\KL(P,Q)$ to denote the relative entropy (or Kullback-Leibler divergence) of $P$ with respect to $Q$, which is infinite when $P$ is not absolutely continuous with respect to $Q$, and otherwise it is defined as $D(P||Q) = \int \log(\frac{dP}{dQ}) dP $, where $dP/dQ$ is the Radon-Nikodym derivative of $P$ with respect to $Q$. 
By abusing notation, we will use $P(X)$ to denote the probability distribution 
$P(X \in \cdot)$
of a random element $X$ under probability measure $P$.
For jointly distributed random elements $X$ and $Y$, we let $P(X|Y)$ denote the conditional distribution of $X$ given $Y$, $P(X \in \cdot | Y)$, which is $Y$-measurable. 
With this, the \emph{chain rule} for relative entropy states that 
\begin{align*}
\KL(P(X,Y),Q(X,Y))=\int \KL(P(X|Y),Q(X|Y)) dP + \KL(P(Y),Q(Y))\,,
\end{align*}
which, of course, extends to any number of jointly distributed random elements. 

We will also need the following result, which is given, for example, as Theorem~14.2 in the book of \citet{lattimore2020bandit}. 
\begin{lemma}[Bretagnolle--Huber inequality]
\label{lem:pinskerhp}
Let $P$ and $Q$ be probability measures on the same measurable space $(\Omega, \cF)$, and let $A \in \cF$ be an arbitrary event. Then,
\begin{align}\label{eq:pinskerhp}
P(A) + Q(A^c) \geq \frac{1}{2} \exp\left(-\KL(P, Q)\right)\,,
\end{align} 
where $A^c = \Omega \setminus A$ is the complement of $A$.
\end{lemma}

\begin{proof}[Proof of Theorem~\ref{thm:pi-induced-lb}]
\revise{
We first prove the result when the logging policy $\logpi$ is a memoryless policy.
}  

We first consider the case where $\rS\geq H:=H_{\gamma, 2\varepsilon}+2$. 
We let $\{s_0,s_1,\dots,s_H,z\}$ be arbitrary, distinct states and choose $\mu$ to be the distribution that is concentrated at $s_0$.
\revise{Fix the logging policy $\logpi$ and the episode lengths $\mathbf{h}=(h_0,\dots,h_{m-1})$
 (we allow these to depend on $\mu$).}
We now define two MDPs, $M_1,M_{-1}\in \cM(S,A)$ (cf. \cref{{fig:lb}}). 
For any $s\in\cS$, let $a_s = \argmin_a \logpi(a|s)$ be the action with the minimal chance of being selected by $\logpi$. 
Note that $\logpi(a_s | s)\leq 1 / \rA$. 

The transition structure in the two MDPs are identical, the transitions are deterministic and all the rewards are also the same with the exception of one transition.
The details are as follows.
State $z$ is absorbing: For any action taken at $z$, the next state is $z$.
For $i<H$, $s_i$ is followed by $s_{i+1}$ when $a_{s_i}$ is taken,
while the next state is $z$ when any other action is taken at this state.
At $s_H$ under any action, the next state  is $z$.
The rewards are deterministically zero for any state-action pair except when the state is $s_H$ and action $a_{s_H}$ is taken at this state. In this case, the reward $R$ is drawn from a Gaussian with mean $\alpha\in \{-1,+1\}$ in MDP $M_\alpha$.

We will use $v_\alpha^\pi$, $v^*_\alpha$ and $\nu_{\mu, \alpha}$
to denote 
the value function of a policy $\pi$ on $M_\alpha$, 
the optimal value function on $M_\alpha$,
and the discounted occupancy measure on $M_\alpha$ with $\mu$ as the initial state distribution, respectively. Note that $v^*_1(s_0) = \gamma^H \ge \gamma^{\frac{\ln(1/(2\epsilon))}{\ln(1/\gamma)}}=2\epsilon$, where the first inequality is because $\gamma\le 1$ and $H\le \frac{\ln(1/(2\epsilon))}{\ln(1/\gamma)}$ by its definition. Note also that $v^*_{-1}(s_0) = 0$.

We now show that if the number of episodes $m$ is too small, then no algorithm will be sound both on $M_1$ and $M_{-1}$. 

For this fix an arbitrary BPO algorithm $\cL$.
Let the data collected by following the logging policy $\logpi$ be
$\cD = (S_i, A_i, R_i, S'_i)_{i=0}^{n-1}$. 
Let $\pi$ be the output of $\cL$.
Let $\sP_\alpha$ be the distribution 
over $(\cD, \pi)$ induced 
by using $\logpi$ on $M_\alpha$ with episode lengths $\mathbf{h}$ and $\mu$ and then running $\cL$ on $\cD$ to produce $\pi$. Note that both $\sP_1$ and $\sP_{-1}$ share the same measure space. Let $\EE_\alpha$ be the expectation operator for $\sP_\alpha$. 

Define the event $E = \{v_1^\pi(s_0)<\varepsilon\}$. Let $E^c$ be the complement of $E$.
Let $a \vee b$ denote the maximum of $a$ and $b$.
We first prove the following claim:

\noindent \underline{Claim:}
If 
\begin{align}
\sP_1(E)\vee \sP_{-1}(E^c)\ge \delta
\label{eq:eeclb}
\end{align}
 then $\cL$ is not $(\epsilon,\delta)$-sound.

\begin{proof}[Proof of the claim]

By \cref{eq:noted},  $\cL$ is not $(\epsilon,\delta)$-sound if 
\begin{align*}
\sP_1( v_1^*(s_0)-v_1^\pi(s_0)\ge \epsilon) \vee \sP_{-1}(v_{-1}^*(s_0)-v_{-1}^\pi(s_0)\ge\epsilon) \ge \delta\,.
\end{align*}
By $v_1^*(s_0)\ge 2\varepsilon$, we have
\begin{align*}
 \sP_1( v_1^*(s_0)-v_1^\pi(s_0)\ge \epsilon)
\ge \sP_1( v_1^\pi(s_0) \le \epsilon)
\ge \sP_1( v_1^\pi(s_0) < \epsilon)
=
\sP_1(E)\,.
\end{align*}
Similarly, by $v_{-1}^*(s_0)=0$, we have
\begin{align*}
 \sP_{-1}( v_{-1}^*(s_0)-v_{-1}^\pi(s_0)\ge \epsilon)
  = \sP_{-1}( v_{-1}^\pi(s_0) \le - \epsilon)
  \ge \sP_{-1}( v_1^\pi(s_0)\ge \epsilon ) = \sP_{-1}(E^c)\,,
\end{align*}
where the inequality follows because if $ v_1^\pi(s_0)\ge \epsilon$ holds then
since $v_1^\pi(s_0) = \langle \nu^\pi_{1}, r_1^\pi \rangle= \nu^\pi_{1}(s_{H},a_{s_{H}}) r_1^\pi(s_{H},a_{s_{H}}) =  \nu^\pi_{1}(s_{H},a_{s_{H}})$ and since
 the transitions in $M_1$ and $M_{-1}$ are same, we have $\nu^\pi_{-1}(s_{H},a_{s_{H}})=\nu^\pi_1(s_{H},a_{s_{H}}) \geq \varepsilon$ and therefore $v_{-1}^\pi(s_0) = -\nu^\pi_{-1}(s_{H},a_{s_{H}})  \leq -\varepsilon$.  
 
Putting things together, we get that 
\begin{align*}
\sP_1( v_1^*(s_0)-v_1^\pi(s_0)\ge \epsilon) \vee \sP_{-1}(v_{-1}^*(s_0)-v_{-1}^\pi(s_0)\ge\epsilon)
\ge
\sP_1(E) \vee \sP_{-1}(E^c)\ge \delta,
\end{align*} 
where the last inequality follows by our assumption.
\end{proof}
It remains to prove that \cref{eq:eeclb} holds.
For this, note that by the Bretagnolle-Huber inequality (\cref{lem:pinskerhp}) we have,
\begin{align}
\sP_1(E)\vee \sP_{-1}(E^c)\geq \frac{\sP_{1}(E) + \sP_{-1}(E^c)}{2} \geq \frac{1}{4}\exp(-\KL(\sP_1, \sP_{-1}))\, .
\label{eq:bh1}
\end{align}
It remains to upper bound $\KL(\sP_1, \sP_{-1})$.
Let $U_0 = S_0$, $U_1 = A_0$, $U_2 = R_0$, $U_3=S_0'$, $U_4=S_1$, $\dots$, $U_{4(n-1)}=S_{n-1}'$.
Further, for $0\le j \le 4n-1$ let $U_{0:j} = (U_0,\dots,U_j)$ and let $U_{0:-1}$ stand for a ``dummy'' (trivial) random element.
By the chain rule for relative entropy,%
\footnote{Here, we use a notation common in information theory, which uses $P(X)$ ($P(X|Y)$) to denote the distribution of $X$ induced by $P$ (the conditional distribution of $X$, given $Y$, induced by $P$, respectively).}
\begin{align*}
\KL(\sP_1, \sP_{-1})
& =
\EE_1[ \KL(\sP_1(\pi|U_{0:4(n-1)}), \sP_{-1}(\pi|U_{0:4(n-1)}))]\\
& +
\sum_{j=0}^{4(n-1)}
\EE_1 [\KL(\sP_1(U_j|U_{0:j-1}),\sP_{-1}(U_j|U_{0:j-1}))]\,.
\end{align*}
Note that, $\sP_1$-almost surely,
$\sP_1(\pi|U_{0:4(n-1)})=\sP_{-1}(\pi|U_{0:4(n-1)})$ since, by definition, $\cL$ assigns a fixed probability distribution over the policies to any possible dataset.
For $0\le j \le 4(n-1)$, let $D_j = \KL(\sP_1(U_j|U_{0:j-1}),\sP_{-1}(U_j|U_{0:j-1}))$.
Since the only difference between $M_1$ and $M_{-1}$ is 
in the reward distribution corresponding to taking action $a_{s_H}$ in state $s_H$,
unless $j=4i+2$ for some $i\in [n]$ and $S_i=s_H,A_i=a_{s_H}$,
we have $D_j=0$ $\sP_1$-almost surely. 
Further, when $j=4i+2$, 
$\sP_1$-almost surely 
we have $D_j = \sI\{S_i = s_{H},A_i=a_{s_{H}}\}{(1 - (-1))^2}/{2} = 2\sI\{S_i = s_{H},A_i=a_{s_{H}}\}$ by the formula for the relative entropy between $\cN(1,1)$ and $\cN(-1,1)$.
Therefore,
\begin{equation*}
\KL(\sP_1, \sP_{-1}) = 2 \EE_1 \left[  \sum_{i=0}^{n-1}  \sI\{S_i = s_{H},A_i=a_{s_{H}}\} \right] \le 2m \sP_1 ( S_{H} = s_{H},A_H=a_{s_{H}} ) \leq \frac{2m}{\rA^{H+1}}\, ,
\end{equation*}
where the first inequality follows from that, by the construction of $M_1$, $s_H$ can be visited only in the $H$th step of \emph{any} episode, the data in distinct episodes are identically distributed, and there are at most $m$ episodes. The second inequality follows because 
\begin{align*}
\MoveEqLeft \sP_1 ( S_{H} = s_{H},A_H=a_{s_{H}} ) 
= \sP_1 ( A_H=a_{s_{H}}|S_{H} = s_{H} ) \sP_1(S_{H} = s_{H}) \\
&= \sP_1 ( A_H=a_{s_{H}}|S_{H} = s_{H} ) \sP_1(A_{H-1}=a_{s_{H-1}},S_{H-1} = s_{H-1}) \\
&= \sP_1 ( A_H=a_{s_{H}}|S_{H} = s_{H} ) 
	  \sP_1(A_{H-1}=a_{s_{H-1}}|S_{H-1} = s_{H-1}) \dots
	  \sP_1(A_{0}=a_{s_{0}}|S_{0} = s_{0}) \\
& = \logpi(a_{s_0}|s_0) \dots \logpi(a_{s_H}|s_H) \le \frac{1}{\rA^{H+1}}\,,	  
\addeq\label{eq:p1-reaching-ub}
\end{align*}
where the last inequality follows by the choice of $a_{s_i}$, $i\in [H+1]$.
Plugging the upper bound on $\KL(\sP_1,\sP_{-1})$ into \cref{eq:bh1},
we get
that 
\begin{align*}
\sP_1(E)\vee \sP_{-1}(E^c)\geq \frac{1}{4}\exp(-2m \rA^{-(H+1)})
\end{align*}
which is larger than $\delta$ if $m\leq{ (\rA^{H+1}\ln\frac{1}{4\delta})}/{2}$. The result then follows by our previous claim.

To prove the result for $\rS< H_{\gamma, 2\varepsilon}+2$, we use the same construction as described above with $H_{\gamma, 2\varepsilon'} = \rS-2 < H_{\gamma, 2\varepsilon}$ for some $\varepsilon' \geq \varepsilon$. Then any learning algorithm $\cL$ needs at least $(\rA^{H_{\gamma, 2\varepsilon'}+1}\ln\frac{1}{4\delta})/{2}$ episodes to be $(\varepsilon', \delta)$-sound. To be $(\varepsilon, \delta)$-sound it needs at least the same amount of data. This finishes the proof.

\revise{
It remains to show  the proof for the case when the logging policy $\logpi$ is  a mixture policy. 
In particular, assume that $\logpi$ is the mixture of the $k$ policies $\logpi_1,\dots,\logpi_k$ with coefficients $p=(p_1,\dots,p_k)\in \Delta([k])$ where $k>0$ is a positive integer.
As before,
consider our MDP with states $\cS = \{s_0,\dots,S_H,z\}$ (cf. \cref{fig:lb}), but now the actions $a_{s_i}$ with $0\le i \le H-1$ are chosen inductively as follows:
We let $a_0 = \argmin_a \sum_j p_j \logpi_j(a|s_0)$.
Together with the action sequence, we simultaneously define $(q_i)$ and $(q_i^j)$ for $0\le i \le H-1$ and $1\le j \le k$.
The intention of these definitions is that 
$q_i^j = \sP^{\logpi_j}(S_i=s_i)$ and
$q_i = \sP^{\logpi}(S_i=s_i)$ should hold where $\sP^{\logpi}$ ($\sP^{\logpi_j}$) is the measure that results from interconnecting either of the two MDPs with $\logpi$ (respectively, with $\logpi_j$)
while the initial state is $s_0$ with probability one.

We let
$q_0=q_0^j=1$ and for $0<i\le H-1$ we let
$q_i = \sum_j p_j q_i^j$
and $q_i^j = q_{i-1}^j \logpi_j(a_{s_{i-1}}|s_{i-1})$,
while choosing $a_{s_i} = \argmin_a \sum_j p_j \logpi_j(a|s_i) q_i^j$.
Then, it follows that 
$a_{s_i}=\argmin_a \sP_1(A_i=a|S_i=s_i)$ provided that $\sP_1(S_i=s_i)>0$ and as such
$\sP_1(A_i=a_{s_i}|S_i=s_i)\leq 1 / \rA$. 
The proof then goes through as before, except that in
\cref{eq:p1-reaching-ub}, 
we directly use that  $\sP_1(A_i=a|S_i=s_i)\le 1/\rA$ for $0\le i \le H$. (The seemingly complex construction is needed because before choosing $a_{s_i}$, the MDP is not defined and as such so are $\sP^{\log}$ and also $\sP_{\pm 1}$.)
}

\end{proof}

\begin{proof}[Proof of Corollary~\ref{cor:warmstartdoesnothelp}]
The result directly follows from the lower bound construction in \cref{thm:pi-induced-lb}. 
\end{proof}

\begin{proof}[Proof of Theorem~\ref{thm:pi-induced-lb-fin}]
This proof is similar to the proof of Theorem~\ref{thm:pi-induced-lb}.
\revise{Assume again initially that $\logpi$ is memoryless.}
We first consider the case where $\rS\geq H+1$. We construct the same MDPs as in the proof of Theorem~\ref{thm:pi-induced-lb} except that the chain consists of $H$ states, that is, ending at $s_{H-1}$ and the hidden reward $R$ is at $(s_{H-1},a_{s_{H-1}})$. The logging policy $\logpi$ collects $m$ trajectories with length $H$ as the dataset $\cD = (S_i, A_i, R_i, S'_i)_{i=0}^{mH-1}$, where $S_0 = S_H=\dots=S_{(m-1)H}=s_0$. 
Now we consider two MDPs $M_{\alpha}\in\cM, \alpha \in\{2\varepsilon, -2\varepsilon\}$, where the reward $R\sim \cN(\alpha, 1)$ on $M_\alpha$.

We use the same notation as in the proof of Theorem~\ref{thm:pi-induced-lb}. Define the event $E = \{v_{2\varepsilon}^\pi(s_0)<\varepsilon\}$. Then, by following the same arguments we can show that $\cL$ is not $(\varepsilon, \delta)$-sound on $M_{2\varepsilon}$ if $\sP_{2\varepsilon}(E)\ge \delta$ and that $\cL$ is not $(\varepsilon, \delta)$-sound on $M_{-2\varepsilon}$ if $\sP_{-2\varepsilon}(E^c)\ge\delta$.

By the Bretagnolle–Huber inequality, we have
\begin{align*}
\max\{\sP_{2\varepsilon}(E), \sP_{-2\varepsilon}(E^c)\} \geq \frac{\sP_{2\varepsilon}(E) + \sP_{-2\varepsilon}(E^c)}{2} \geq \frac{1}{4}\exp(-\KL(\sP_{2\varepsilon}, \sP_{-2\varepsilon}))\, .
\end{align*}

Similarly as in the proof of Theorem~\ref{thm:pi-induced-lb}, we obtain
\begin{align*}
\KL(\sP_{2\varepsilon}, \sP_{-2\varepsilon}) & = 8\varepsilon^2 \EE_{2\varepsilon} \left[  \sum_{i=0}^{mH-1}  \sI\{S_i = s_{H-1},A_i=a_{s_{H-1}}\} \right]\\
&= 8m\varepsilon^2 \EE_{2\varepsilon} \left[  \sum_{i=0}^{H-1}  \sI\{S_i = s_{H-1},A_i=a_{s_{H-1}}\} \right]  \\
 & = 8m\varepsilon^2 \sP_{2\varepsilon} ( S_{H-1} = s_{H-1},A_{H-1}=a_{s_{H-1}} ) \leq \frac{8m\varepsilon^2}{\rA^H}\, ,
\end{align*}
where the second equality is obtained by the fact that the episodes are independently sampled. 
Combining the above together we have that if $m\leq\frac{ \rA^H\ln\frac{1}{4\delta}}{8\varepsilon^2}$, $\max\{\sP_{2\varepsilon}(E), \sP_{-2\varepsilon}(E^c)\}\geq \delta$, which means that $\cL$ is not $(\varepsilon, \delta)$-sound on either $M_{2\varepsilon}$ or $M_{-2\varepsilon}$. 

To prove the result for $\rS\leq H$, we use the same construction as described above with $H' = \rS-1 < H$. Then any learning algorithm $\cL$ needs at least $\frac{ \rA^{H'}\ln\frac{1}{4\delta}}{8\varepsilon^2}$ trajectories to be $(\varepsilon, \delta)$-sound. This finishes the proof.  

\revise{
Finally, if $\logpi$ is a mixture policy, we can reuse the same argument as in the proof of \cref{thm:pi-induced-lb} to construct the MDPs and obtain the same result as above.}

\end{proof}

\begin{proof}[Proof of Theorem~\ref{thm:pi-induced-lb-avg}]
We use MDPs similar to those in the proof of Theorem~\ref{thm:pi-induced-lb} but with some key differences. Let the state space consist of three parts $\cS = \{s_0, s_1, \dots, s_{H-1}\} \cup \{y\}\cup \{z\}$, where $H=\rS-2$. Consider $\mu$ concentrated on $s_0$. For any $s\in\cS$, let $a_s = \argmin_a \logpi(a|s)$. At $s_i$ for $i\in\{0, \dots, H-2\}$, it transits to $s_{i+1}$ by taking $a_{s_i}$ and transits to $z$ by taking any other actions, where $z$ is an absorbing state. At $s_{H-1}$, by taking any action it transits to $y$ with probability $p>0$ and goes back to $s_0$ with probability $1-p$. $y$ is also an absorbing state, but there is a reward $R$ for any action in $y$. The rewards are deterministically zero for any other state-action pairs. 

Now consider two such MDPs $M_{\alpha}\in\cM, \alpha \in\{2\varepsilon, -2\varepsilon\}$, where the reward $R\sim \cN(\alpha, 1)$ on $M_\alpha$. 
We keep using the same notation $v_\alpha$ and $\nu_{\mu, \alpha}$, the latter of which denotes the occupancy measure on $M_\alpha$ with $\mu$ as the initial state distribution. Also, we use the rest of notation in the proof of Theorem~\ref{thm:pi-induced-lb}. Recall that $\pi$ is the output policy of a learning algorithm $\cL$.   

Define the event $E = \{v_{2\varepsilon}^\pi(s_0)<\varepsilon\}$.

\noindent \underline{Claim:}
If 
\begin{align}
\sP_{2\varepsilon}(E)\vee \sP_{-2\varepsilon}(E^c)\ge \delta
\label{eq:eeclbThm3}
\end{align}
 then $\cL$ is not $(\epsilon,\delta)$-sound.
\begin{proof}[Proof of the claim]
By \cref{eq:noted},  $\cL$ is not $(\epsilon,\delta)$-sound if 
\begin{align*}
\sP_{2\varepsilon}( v_{2\varepsilon}^*(s_0)-v_{2\varepsilon}^\pi(s_0)\ge \epsilon) \vee \sP_{-2\varepsilon}(v_{-2\varepsilon}^*(s_0)-v_{-2\varepsilon}^\pi(s_0)\ge\epsilon) \ge \delta\,.
\end{align*}

By the definition of $M_{2\varepsilon}$, the optimal policy is choosing $a_{s_i}$ at $s_i$ for $i\in\{0, \dots, H-2\}$. We have $v_{2\varepsilon}^*(s_0)=2\varepsilon$, because $p$ is positive, and thus, the optimal policy reaches $y$ in finite steps with probability one. 
Thus, we have
\begin{align*}
 \sP_{2\varepsilon}( v_{2\varepsilon}^*(s_0)-v_{2\varepsilon}^\pi(s_0)\ge \epsilon)
= \sP_{2\varepsilon}( 2\varepsilon - v_{2\varepsilon}^\pi(s_0)\ge \epsilon)
\ge \sP_{2\varepsilon}( v_{2\varepsilon}^\pi(s_0) < \epsilon)
= \sP_{2\varepsilon}(E)\,.
\end{align*}
Similarly, by $v_{-2\varepsilon}^*(s_0)=0$, we have
\begin{align*}
 \sP_{-2\varepsilon}( v_{-2\varepsilon}^*(s_0)-v_{-2\varepsilon}^\pi(s_0)\ge \epsilon)
  = \sP_{-2\varepsilon}( v_{-2\varepsilon}^\pi(s_0) \le - \epsilon)
  \ge \sP_{-2\varepsilon}( v_{2\varepsilon}^\pi(s_0)\ge \epsilon ) = \sP_{-2\varepsilon}(E^c)\,,
\end{align*}
where the inequality follows because if $ v_{2\varepsilon}^\pi(s_0)\ge \epsilon$ holds then since $v_{2\varepsilon}^\pi(s_0) = \langle \nu^\pi_{2\varepsilon}, r_{2\varepsilon}^\pi \rangle=2\varepsilon\nu^\pi_{2\varepsilon}(y)$ and since
 the transitions in $M_{2\varepsilon}$ and $M_{-{2\varepsilon}}$ are same, we have $\nu^\pi_{-2\varepsilon}(y)=\nu^\pi_{2\varepsilon}(y) \geq 1/2$ and therefore $v_{-2\varepsilon}^\pi(s_0) = -2\varepsilon\nu^\pi_{-2\varepsilon}(y)  \leq -\varepsilon$.  

Putting things together, we get that 
\begin{align*}
\sP_{2\varepsilon}( v_{2\varepsilon}^*(s_0)-v_{2\varepsilon}^\pi(s_0)\ge \epsilon) \vee \sP_{-2\varepsilon}(v_{-2\varepsilon}^*(s_0)-v_{-2\varepsilon}^\pi(s_0)\ge\epsilon)
\ge
\sP_{2\varepsilon}(E) \vee \sP_{-2\varepsilon}(E^c)\ge \delta,
\end{align*} 
where the last inequality follows by our assumption.
\end{proof}

Following the same arguments in the proof of Theorem~\ref{thm:pi-induced-lb}, we have
\begin{align*}
\KL(\sP_{2\varepsilon}, \sP_{-2\varepsilon}) &= 8\varepsilon^2\EE_{2\varepsilon} \left[  \sum_{i=0}^{n-1}  \sI\{S_i = y\} \right] = 8\varepsilon^2\sum_{i=0}^{n-1}\sP_{2\varepsilon}\{S_i = y\}\\ 
&=8\varepsilon^2p\sum_{i=1}^{n-1}\sP_{2\varepsilon}\{S_{i-1} = s_{H-1}\}\leq  \frac{8\varepsilon^2np}{\rA^{H-1}}\, .
\end{align*}
Combining the above together and using the Bretagnolle–Huber inequality (\cref{lem:pinskerhp}) as we did in the proof of Theorem~\ref{thm:pi-induced-lb}, we have that if $n\leq \frac{ \rA^{H-1}\ln\frac{1}{4\delta}}{8\varepsilon^2p}$, then $\cL$ is not $(\varepsilon, \delta)$-sound on either $M_{2\varepsilon}$ or $M_{-2\varepsilon}$. We obtain the result by sending $p$ to zero from the right hand side.  
\end{proof}

For the proof of \cref{thm:pi-free-lb}, we will need some results on the relative entropy between Bernoulli distributions, which we present now.
\newcommand{\Ber}{\text{Ber}}
Let $\Ber(p)$ denote the Bernoulli distribution with parameter $p\in [0,1]$. 
As it is well known (and not hard to see from the definition),
\begin{align*}
D(\Ber(p),\Ber(q)) = d(p,q)
\end{align*}
where $d(p,q)$ is the so-called \emph{binary relative entropy function}, which is defined as
\begin{align*}
d(p,q) = p\log(p/q) + (1-p) \log( (1-p)/(1-q))\,.
\end{align*}

\begin{proposition}
\label{prop:ber}
For $p,q\in (0,1)$, defining $p^*$ to be $p$ or $q$ depending on which is further away from $1/2$,
\begin{align}
d(p,q) \le \frac{(p-q)^2}{2p^*(1-p^*)}\,.
\label{eq:dpq}
\end{align}
\end{proposition}
\begin{proof}
Let $R$ be the unnormalized negentropy over $[0,\infty)^2$. Then, by Theorem 26.12 of the book of \citet{lattimore2020bandit}, for any $x,y\in (0,\infty)^2$,
\begin{align*}
D_R(x,y) = \frac{1}{2}\| x- y\|_{\nabla R(z)}^2
\end{align*}
for some $z$ on the line segment connecting $x$ to $y$.
We have $R(z) = z_1 \log(z_1) + z_2 \log(z_2) - z_1 - z_2$. Hence, $\nabla R(z) = [\log(z_1),\log(z_2)]^\top$ and $\nabla R(z) = \text{diag}(1/z_1,1/z_2)$, both defined for $z\in (0,\infty)^2$.
Thus,
\begin{align*}
D_R(x,y) = \frac{(x_1-y_1)^2}{2 z_1} + \frac{(x_2-y_2)^2}{2 z_2}\,.
\end{align*}
Now choosing $x=(p,1-p)$, $y=(q,1-q)$, we see that $x,y\in (0,\infty)^2$ if $p,q\in (0,1)$.
In this case, with some $\alpha\in [0,1]$,
 $z = \alpha x + (1-\alpha) y 
= (\alpha p + (1-\alpha) q,  \alpha(1-p)+(1-\alpha)(1-q))^\top
= (\alpha p + (1-\alpha) q,  1-(\alpha p+(1-\alpha)q) )^\top$.
Hence, $z_2 = 1-z_1$ and
\begin{align*}
d(p,q) = \frac{(p-q)^2}{2z_1} + \frac{(p-q)^2}{2(1-z_1)} = \frac{(p-q)^2}{2 z_1(1-z_1)}\,.
\end{align*}
Now, $z_1(1-z_1)\ge p^*(1-p^*)$ (the function $z\mapsto z(1-z)$ has a maximum at $z=1/2$ and is decreasing on ``either side'' of the line $z=1/2$).
Putting things together, we get 
\begin{align*}
d(p,q) = \frac{(p-q)^2}{2 z_1(1-z_1)} \le \frac{(p-q)^2}{2 p^*(1-p^*)}\,.
\end{align*}
\end{proof}

Now we re-state Theorem~\ref{thm:pi-free-lb} and prove it. 
\begin{theorem}[Restatement of Theorem~\ref{thm:pi-free-lb}]
\label{thm:pi-free-lb-new}
Fix any $\gamma_0>0$.
Then, there exist some constants $c_0,c_1>0$ such that for any
$\gamma\in [\gamma_0,1)$,
any positive integers $S$ and $A$, 
$\delta\in(0,1)$, and $0<\varepsilon\le c_0/(1-\gamma)$,
 the sample size $n$ needed by
any $(\varepsilon,\delta)$-sound algorithm that produces as output a memoryless policy and works with $SA$-sampling for MDPs with $\rS$ states and $\rA$ actions under the $\gamma$-discounted expected reward criterion must be so that 
is at least $c_1 \frac{\rS\rA\ln (1 / (4\delta))}{\varepsilon^2 (1-\gamma)^3}$.
\end{theorem}

\begin{proof}[Proof of \cref{thm:pi-free-lb-new}]
The proof also uses Le Cam's method, just like \cref{thm:pi-induced-lb}.
At the heart of the proof is a gadget with a self-looping state which was introduced by \citet{azar2013minimax} to give a lower bound on the sample complexity of estimating the optimal value function in the simulation setting where the estimate's error is measured with its worst-case error.

\begin{figure}[t]
\begin{center}
\includegraphics[width=8cm]{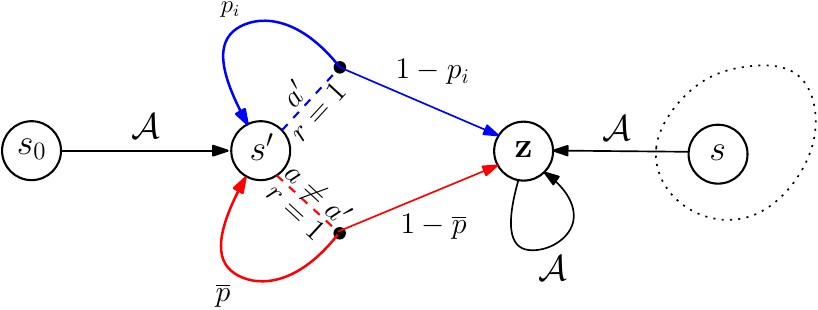}
\end{center}
\caption{
Illustration of the MDPs used in the proof of \cref{thm:pi-free-lb}. 
The initial distribution $\mu$ concentrates on state $\{s_0\}$. The pair $(s',a')$ is the one where $\mu_{\log}$ takes on the smallest value (which is below $1/(\rS \rA)$) and without loss of generality $s'\ne s_0$ and taking any action in $s_0$ makes the next state $s'$.
We have $p_0< \bar p < p_1$, all in the $[1/2,1)$ interval. In MDP $M_i$ with $i\in \{0,1\}$, the probability of transitioning under action $a'$ from $s'$ to $z$, an absorbing state, is $p_i$, while with probability $1-p_i$, the next state is $s'$. All other actions use probability $\bar p$ at this state. All other states under any action lead to $z$.
The rewards are deterministically zero except at state $s'$, when all actions yield a reward of one, regardless of the identity of the next state.
}  
\label{fig:salb}
\end{figure}

The idea of the proof is illustrated by \cref{fig:salb}.
Fix an initial state distribution $\mu$ concentrated on an arbitrary state $s_0\in \cS$. 
Let $\mu_{\log}$ be the logging distribution chosen based on $\mu$ and 
let $(s',a')$ be any state-action pair that has the minimum sampling probability under $\mu_{\text{log}}$. Note that $\mu_{\text{log}}(s',a')\le 1/(\rS\rA)$. 
Assume that $s'\ne s_0$. As we shall see by the end of the proof, there is no loss of generality in making this assumption (when $s'=s_0$, the lower bound would be larger).

We construct two MDPs as follows. 
The reward structures of the two MDPs are completely identical and the transition structures are also identical except for when action $a'$ is taken at state $s'$.
In particular, in both MDPs,
the rewards are identically zero except at state $s'$, where for any action the reward incurred is one.
The transition structures are as follows: Let $p_0<\bar{p}<p_1$ be in $(0,1)$, to be determined later.
At $s_0$, by taking any action the system transits to $s'$ deterministically. 
At $s'$, for any $a\in\cA\setminus\{a'\}$, taking action $a$ leads to $s'$ as the next state with probability $\bar{p}$
 and to $z$ with probability $1-\bar{p}$, where $z$ is an absorbing state. 
The transition under $a'$ at $s'$ is similar, except that in $M_i$ ($i\in \{0,1\}$), the probability that the next state is $s'$ is $p_i$ (and the probability that the next state is $z$ is $1-p_i$).
 At any state $s\in\cS\setminus\{s_0,s'\}$, taking any action moves the system to $z$ deterministically. 
The optimal policy at $s'$ in $M_1$ is to pull the action $a'$, while in $M_0$ all the other actions are optimal. 
It is easy to see that in any of these MDPs, for $a\in\cA$, 
\begin{align*}
q^*(s',a) = \frac{1}{ 1 - \gamma P(s'|s',a)}\ .
\end{align*}

Now we select $p_0,\bar{p}$, and $p_1$. Let $b$ be a constant such that $1<b<\frac{1-\gamma/2}{1-\gamma}$. We will choose a specific value for $b$ at the end of the proof. The values of $p_0,\tilde p$ and $p_1$ will depend on $b$. In particular, we choose 
$\tilde{p}\in (1/2,1)$ so that
\begin{equation*}
    b=\frac{1-\gamma \tilde{p}}{1-\gamma},
\end{equation*}
while we set $p_0=\tilde{p}$.
Then $p_0=(1-b+\gamma b)/\gamma$. Note that $p_0>1/2$ by its choice. Let $f(p) = \frac{\gamma}{ 1 - \gamma p}$. 
Note that for any deterministic policy $\pi$, $v^\pi(\mu)=f( P(s'|s',\pi(s')) )$
and also $v^*(\mu)=f( \bar p)$ in MDP $M_0$ and $v^*(\mu)=f(p_1)$ in MDP $M_1$.

By Taylor's theorem, for some $p\in[p_0, p_1]$, we have
\begin{align*}
f(p_1) = f(p_0) + f'(p)(p_1 - p_0) \geq f(p_0) + f'(p_0)(p_1 - p_0)\, ,
\end{align*}
where the inequality follows by $p_1> p_0$ and the fact that $f'$ is increasing. Thus, if $p_1 - p_0 \geq 4\varepsilon / f'(p_0)$, we have $f(p_1)  \geq f(p_0) + 4\varepsilon$. 
Because of the choice of $p_0$,
\begin{align*}
f'(p_0) = \frac{\gamma^2}{(1-\gamma p_0)^2}=\frac{\gamma^2}{(1-\gamma)^2b^2}\, .
\end{align*}
We let $p_1 = p_0 + 4\varepsilon / f'(p_0)$. Then, we have $p_1< 1$ given that $\varepsilon \leq c_0/(1-\gamma):=\frac{\gamma(b-1)}{8(1-\gamma)b^2}$, because 
\begin{align}
p_1 - 1 &= p_0 + 4\varepsilon / f'(p_0) - 1 = \frac{1-b+\gamma b}{\gamma} + \frac{4(1-\gamma)^2\varepsilon b^2}{\gamma^2} - 1\notag\\
&= \frac{4(1-\gamma)^2\varepsilon b^2+\gamma(\gamma-1)(b-1)}{\gamma^2}\leq \frac{(\gamma-1)(b-1)}{2\gamma}< 0\, ,\label{eq:Thm4proof1}
\end{align} 
where the first inequality is due to the choice of $\varepsilon$. Lastly, we set $\bar{p}$ so that $f(\bar{p})=[f(p_0)+f(p_1)]/2$ (such $\bar{p}$ uniquely exists because $f$ is increasing and continuous). Note that $f(p_1)-f(\bar{p})\geq 2\varepsilon$ and $f(\bar{p})-f(p_0)\geq 2\varepsilon$. 

Let $\sP_0$ and $\sP_1$ be the joint probability distribution on the data and the output policy of any given learning algorithm $\cL$, induced by $\mu$, $\mu_{\text{log}}$, $\cL$, and the two MDPs $M_0$ and $M_1$, respectively. For any algorithm $\cL$, let $E=\{\pi(a'|s')\ge 1/2\}$, where $\pi$ is the output of $\cL$. 

If $E$ is true, in $M_0$,
\begin{align*}
    v^\pi(\mu)=\pi(a'|s')f(p_0)+(1-\pi(a'|s'))f(\bar{p})\le \frac{f(p_0)+f(\bar{p})}{2}\le \frac{(f(\bar{p})-2\varepsilon)+f(\bar{p})}{2}=f(\bar{p})-\varepsilon\ .
\end{align*}
Thus, $\cL$ is not $(\varepsilon,\delta)$-sound for $M_0$ if $\sP_0(E)\geq \delta$.
If $E^c$ holds, in $M_1$,
\begin{align*}
    v^\pi(\mu)=\pi(a'|s')f(p_1)+(1-\pi(a'|s'))f(\bar{p})\le \frac{f(p_1)+f(\bar{p})}{2}\le \frac{f(p_1)+(f(p_1)-2\varepsilon)}{2}=f(p_1)-\varepsilon\ .
\end{align*}
Therefore, if $\sP_1(E^c)\geq \delta$, then $\cL$ is not $(\varepsilon,\delta)$-sound for $M_1$. 

By the Bretagnolle-Huber inequality  (\cref{lem:pinskerhp}) we have,
\begin{equation*}
\max\{ \sP_0(E),  \sP_1 (E^c)\} \geq \frac{\sP_0(E) + \sP_1 (E^c)}{2} \geq \frac{1}{4} \exp\left( - \KL(\sP_0 || \sP_1) \right)\,.
\end{equation*}

Recall that $n$ is the number of samples. Since $M_0$ and $M_1$ differ only in the state transition from $(s',a')$, by the chain rule of relative entropy,
with a reasoning similar to that used
in the proof of Theorem~\ref{thm:pi-induced-lb},
we derive
\begin{align*}
\KL(\sP_0, \sP_1) & = n  \sP_0(S_i=s',A_i=a')
\KL(\text{Ber}(p_0), \text{Ber}(p_1))\leq \frac{n}{\rS\rA}\cdot \frac{(p_0 - p_1)^2}{2p_1(1-p_1)} \\
&= \frac{n}{\rS\rA}\cdot \frac{16\varepsilon^2 (1-\gamma)^4b^4}{2\gamma^4p_1(1-p_1)} \\
& < \frac{n}{\rS\rA}\cdot \frac{16\varepsilon^2(1-\gamma)^3b^4 }{\gamma^3(b-1)p_0}\ ,
\end{align*}
where the first inequality is due to \cref{prop:ber} and the second inequality is due to \cref{eq:Thm4proof1} and the fact that $p_0<p_1$. 

Now fix $\gamma_0\in (0,1)$ and 
let $\gamma\geq \gamma_0$
and choose 
$b=0.5(1+\tfrac{1-\gamma_0/2}{1-\gamma_0}) \in (1,\tfrac{1-\gamma/2}{1-\gamma})$.
Then, combining the above together and reordering show that if $n\leq c_1 \frac{\rS\rA\ln (1 / (4\delta))}{\varepsilon^2 (1-\gamma)^3}$ where $c_1=\frac{\gamma_0^3(b-1)p_0}{16b^4}$, we can guarantee that $\cL$ is not $(\varepsilon, \delta)$-sound on either $M_0$ or $M_1$, concluding the proof.
\end{proof}

\section{Upper bound proofs}

We start with some extra notation. 
We identify the transition function $P$ as an $\rS\rA\times\rS$ matrix, whose entries $P_{sa, s'}$ specify the conditional probability of transitioning to state $s'$ starting from state $s$ and taking action $a$, 
and the reward function $r$ as an $\rS\rA\times1$ reward vector. 
We use $\| x \|_1$ to denote the 1-norm $\sum_i |x_i|$ of $x\in\RR^n$.

%
Recall first that we defined 
$P^\pi$ to be the transition matrix on state-action pairs induced by the policy $\pi$. 
Define the $H$-step action-value function for $H > 0$ by
\begin{align*}
q^\pi_H = \sum_{h=0}^{H-1} (\gamma P^\pi)^h r\, .
\end{align*}
We let $v^\pi_H$ denote the $H$-step state-value function. 
In what follows we will need quantities for $\hat{M}$, which, in general could be any MDP 
that differs from $M$ from only its transition kernel.
Quantities related to $\hat{M}$ receive a ``hat''. For example, we use $\hat{P}$ for the transition kernel of $\hat{M}$, 
$\hat{P}^\pi$  for the state-action transition matrix induced by a policy $\pi$ and $\hat{P}$, etc.

In subsequent proofs, we will need the following lemma, which gives two decompositions of the difference between the action-value functions on two MDPs, $M$ and $\hat{M}$:
\begin{lemma}
For any policy $\pi$, transition model $\hat P$, and $H>0$, 
\begin{align}
q^\pi_H - \hat{q}^\pi_H &= \gamma \sum_{h=0}^{H-1} (\gamma P^\pi)^h (P - \hat{P} ) \hat{v}^\pi_{H-h-1}\, , \label{eq:ed1}\\
\hat{q}^\pi_H - {q}^\pi_H &= \gamma \sum_{h=0}^{H-1} (\gamma \hat{P}^\pi)^h (\hat{P} - {P} ) {v}^\pi_{H-h-1}\, .  \label{eq:ed2}
\end{align}
\label{lem:H-step-error-decompose}
\end{lemma}
\begin{proof}
By symmetry, it suffices to prove \cref{eq:ed1}. 
For convenience, we re-express a policy $\pi$ as an $\rS\times\rS\rA$ matrix/operator $\Pi$.
In particular, as a left linear operator, $\Pi$ maps $q\in\RR^{\rS \rA}$ to $\sum_{a}\pi(a|\cdot) q(\cdot,a)\in \RR^{\rS}$. 
Note that with this $P^\pi = P\Pi$,  $\hat{P}^\pi = \hat{P}\Pi$, 
$v_h^\pi = \Pi q_h^\pi$ and 
$\hat{v}_h^\pi = \Pi \hat{q}_h^\pi$. 
To reduce clutter, as $\pi$ is fixed, for the rest of the proof we drop the upper indices and just use $v_h$, $\hat{v}_h$, $q_h$ and $\hat{q}_h$.

Note that for $H>0$,
\begin{align*}
q_H & = r + \gamma P \Pi q_{H-1}\,, \qquad \text{and} \\
\hat{q}_H & = r + \gamma \hat{P} \Pi \hat{q}_{H-1}\,.
\end{align*}
Hence,
\begin{align*}
q_H - \hat q_H 
& = \gamma ( P \Pi q_{H-1} - \hat{P} \Pi \hat{q}_{H-1} ) \\
& = \gamma ( P -\hat{P} ) \Pi \hat{q}_{H-1} + \gamma P \Pi (q_{H-1}-\hat q_{H-1})\,.
\end{align*}
Then using $\Pi \hat{q}_{H-1} = \hat{v}_{H-1}$ and recursively expanding $q_{H-1}-\hat{q}_{H-1}$ in the same way gives the result,
noting that $q_0 = r = \hat q_0$.
\end{proof}

\newcommand{\one}[1]{\mathbb{I}\{#1\}}
\newcommand{\E}{\mathbb{E}}
\newcommand{\Prob}[1]{\mathbb{P}\left(#1\right)}

We need two standard results from the concentration of binomial random variables.

\begin{lemma}[Multiplicative Chernoff Bound for the Lower Tail, Theorem~4.5 of \citet{MiUp05:book}]
\label{lem:chernoff-lower}
Let $X_1,\ldots,X_n$ be independent Bernoulli random variables with parameter $p$, $S_n = \sum_{i=1}^n X_i$.
Then, for any $0\le \beta<1$,
\[
\Prob{\frac{S_n}{n} \le (1-\beta)  p } \le \exp\left( - \frac{\beta^2 n p }{2} \right)\,.
\]
\end{lemma}

\begin{lemma}
\label{lem:ber}
Let $n$ be a positive integer, $p>0$, $\delta\in (0,1)$ such that
\begin{align}\label{eq:deltalimit}
\frac{2}{np}\ln\frac{1}{\delta}\le \frac{1}{4}\,.
\end{align}
Let $S_n$ be as in the previous lemma, $\hat{p}=S_n/n$. 
Then,  with probability at least $1-\delta$, it holds that 
\begin{align*}
\hat{p} \ge p/2>0
\end{align*}
while we also have
\begin{align*}
\frac{1}{\hat{p}} \le \frac{1}{p} + \frac{2}{p} \sqrt{\frac{2}{np} \ln\frac{1}{\delta}} \,.
\end{align*}
on the same $(1-\delta)$-probability event.
\end{lemma}
In what follows, we will only need the first lower bound, $\hat p \ge p/2$ from above; the second is useful to optimize constants only.
\begin{proof}
According to the multiplicative Chernoff bound for the low tail (cf. \cref{lem:chernoff-lower}),
for any $0<\delta\le 1$, 
with probability at least $1-\delta$, we have
\[
\hat{p} \ge p - \sqrt{\frac{2p}{n}\ln\frac{1}{\delta}} .
\]
Denote by  $\mathcal{E}_\delta$  the event when this inequality holds. 
Using \cref{eq:deltalimit},
on $\mathcal{E}_\delta$ 
we have
\begin{align*}
\hat{p} \ge p - \sqrt{\frac{2p}{n}\ln\frac{1}{\delta}}
=p \left(1-\sqrt{\frac{2}{pn}\ln\frac1\delta}\right) \ge p\left(1-\frac12\right) = \frac{p}2>0\,,
\end{align*}
and thus,
thanks to $1/(1-x) \le 1+2x$ which holds for any $x\in [0,1/2]$,
\begin{align*}
\frac{1}{\hat{p}} 
 \le  \frac{1}{p} \frac{1}{1-\sqrt{\frac{2}{np} \ln\frac{1}{\delta}}} 
 \le  \frac{1}{p}+ \frac{2}{p} \sqrt{\frac{2}{np} \ln\frac{1}{\delta}} \,.
\end{align*}
\end{proof}

Our next lemma bounds the deviation between the empirical transition kernel and the ``true'' one:
\begin{lemma}
\label{lem:empdev}
With probability $1-\delta$, for any $(s,a)\in \cS \times \cA$,%
\begin{align}
\norm{ \hat{P}(\cdot|s,a) - P(\cdot|s,a) }_1 \le
\beta(N(s,a),\delta)
\label{eq:pdiff}
\end{align}
where for $u\ge 0$,
\begin{align*}
\beta(u,\delta)= 2 \sqrt{\frac{\rS\ln 2 + \ln\frac{u_+(u+1)\rS\rA}\delta}{2 u_+}}\,,
\end{align*}
where $u_+ = u \vee 1$.
\end{lemma}
\begin{proof}
By abusing notation, for $u\ge 0$, let $\beta(u) = 
 2 \sqrt{\frac{\rS\ln 2 + \ln\frac{u_+(u+1)}\delta}{2 u_+}}$, where $u_+ = u \vee 1$.
We will prove below the following claim:

\noindent \underline{Claim:}
For any fixed state-action pair $(s,a)$, with probability $1-\delta$, 
\begin{align*}
\norm{ \hat{P}(\cdot|s,a) - P(\cdot|s,a) }_1 \le \beta( N(s,a) )\,.
\end{align*}

Clearly, from this claim the lemma follows by a union bound over the state-action pairs.
Hence, it remains to prove the claim.

For this fix $(s,a)\in \cS \times \cA$.
Recall that the data
$\cD = ((S_i,A_i,R_i,S_i')_{i\in [n]})$
 that is used to estimate $\hat P(\cdot|s,a)$ consists of $m$ trajectories of length $H$ obtained by following the logging policy $\logpi$ while the initial state is selected from $\mu$ at random.
In particular, for $j\in [m]$, the $j$th trajectory is
$( S_0^{(j)},  A_0^{(j)},  R_0^{(j)}, \dots,  S_{h_j-1}^{(j)},  A_{H-1}^{(j)},  R_{H-1}^{(j)},  S_{H}^{(j)})$, where $ S_0^{(j)}\sim\mu$, $ A_t^{(j)}\sim \logpi(\cdot| S_{t}^{(j)})$, $(R_t^{(j)}, S_{t+1}^{(j)}) \sim Q(\cdot |  S_t^{(j)},  A_t^{(j)})$. 
Clearly, if $q:=\Prob{ \exists 0\le i \le H-1\,:\, S_i^{(0)}=s,A_i^{(0)}=a}=0$ then $N(s,a)=0$ holds
with probability one.
The claim then follows since when $N(s,a)=0$, $\hat{P}(\cdot|s,a)$ is identically zero, hence,
\begin{align}
\norm{ \hat{P}(\cdot|s,a) - P(\cdot|s,a) }_1 = \norm{P(\cdot|s,a)}_1 = 1 \le 1.177\ldots \le \beta(0)\,.
\label{eq:nzero}
\end{align}

Hence, it remains to prove the claim for the case when $q>0$, which we assume from now on.
For convenience, append to the data infinitely many further trajectories, giving rise to the infinite sequence $(S_i,A_i,R_i,S_i')_{i\ge 0}$.
Let $\tau_0=0$ and 
for $u\ge 1$, let $\tau_u = \min\{ i\in \N \,:\, i > \tau_{u-1} \text{ and } S_i=s, A_i=a \}$ be the ``time'' indices when $(s,a)$ is visited,
where we define the minimum of an empty set to be infinite.
Since $q>0$, almost surely $(\tau_u)_{u\ge 0}$ is a well-defined sequence of \emph{finite} random variables.
Now let $X_u = S_{\tau_u}'$ be the ``next state'' upon the $u$th visit of $(s,a)$.
Let $\hat{p}_u(s') = \frac{\sum_{v=1}^u\one{X_v=s'}}{u}$. Note that 
\begin{align}
\hat{P}(\cdot|s,a)=\hat{p}_{N(s,a)}(\cdot)
\label{eq:pup}
\end{align}
provided that $N(s,a)>0$.
By the Markov property, it follows that $(X_v)_{v\ge 1}$ is an i.i.d. sequence of categorical variables with common distribution $p(\cdot):=P(\cdot|s,a)$.

Now,
\newcommand{\ip}[1]{\langle #1 \rangle }
\begin{align*}
\norm{ \hat{p}_u - p }_1 = \max_{y\in \{-1,+1\}^{\cS}} \ip{\hat{p}_u-p, y}\,,
\end{align*}
while
\begin{align*}
 \ip{\hat{p}_u-p, y} = \frac{1}{u} \sum_{v=1}^u \underbrace{y(X_v) - \sum_{s'} p(s') y(s')}_{\Delta_v}\,.
\end{align*}
Now, $(\Delta_v)_{1\le v \le u}$ is an i.i.d. sequence, $|\Delta_v|\le 2$ for any $v$ and $\EE{\Delta_v}=0$.
Hence, by Hoeffding's inequality, with probability $1-\delta$,
\begin{align*}
\frac{1}{u} \sum_{v=1}^u \Delta_v \le 2 \sqrt{\frac{\ln\frac1\delta}{2u}}\,.
\end{align*}
Since the cardinality of $\{-1,+1\}^{\cS}$ is $2^{\rS}$,
applying a union bound over $y\in \{-1,+1\}^{\cS}$, we get that with probability $1-\delta$,
\begin{align*}
\norm{ \hat{p}_u - p }_1 \le 2 \sqrt{\frac{\rS\ln 2 + \ln\frac1\delta}{2u}}\,.
\end{align*}
Applying another union bound over $u$, owing to that $\sum_{u=1}^{\infty} \frac{1}{u(u+1)}=1$,
we get that with probability $1-\delta$, for any $u \ge 1$,
\begin{align*}
\norm{ \hat{p}_u - p }_1 \le 2 \sqrt{\frac{\rS\ln 2 + \ln\frac{u(u+1)}\delta}{2u}} = \beta(u)\,.
\end{align*}
Since $\norm{ \hat{p}_0 - p }_1 \le 1 \le  \beta(0)$ (cf. \cref{eq:nzero}), the claim follows by \cref{eq:pup}.
\end{proof}

We now state a lemma that bounds, with high probability, the error of predicting the value of some fixed policy
when the prediction is based on a transition kernel $P'$ which is ``close'' to the true transition kernel $P$, where closedness is based on how often the individual state-action pairs have been visited. This notion of closedness is motivated by \cref{lem:empdev}; this lemma can be used when $P' = \hat{P}$, or some other transition kernel in the vicinity of $\hat{P}$. The former will be needed in the analysis of the plug-in method presented here; while the latter will be used in the next section where we analyze the pessimistic algorithm.
\begin{lemma}
\label{lem:fixed-pi-subopt-gap}
Let $\delta\in(0,1)$ and $m$ be the number of episodes collected by the logging policy and fix any policy $\pi$. 
For any $P'$ such that for any $(s,a)\in\cS\times\cA$,
\begin{align*}
\Vert P'(\cdot | s,a) - P(\cdot | s,a ) \Vert_1 \leq \frac{C}{\sqrt{N(s,a) \vee 1}},
\end{align*}
with probability at least $1-\delta$ for $C>0$,
 we have
\begin{align}
\label{eq:vdbd}
v^\pi(\mu) - {v}_{P'}^\pi(\mu)
\le
	\frac{4\gamma C \rS \rA^{\frac{\min(H, \rS)}{2} + 1}}{(1-\gamma)^{2}\sqrt{m}} 
	 +  \frac{8\gamma \rS \rA}{(1-\gamma)^2} \frac{\ln \frac{\rS \rA}{\delta} }{m}
	 + \varepsilon \, .
\end{align}
\end{lemma}

\begin{proof}
Note that 
\begin{align*}
\gamma^{H_{\gamma,\epsilon}} \le \gamma^{1 + \frac{\ln(1/\epsilon)}{\ln(1/\gamma)}} = \gamma \epsilon\,.
\end{align*}
Hence, for $H=H_{\gamma, (1-\gamma)\varepsilon/(2\gamma)}$,
\begin{align*}
\gamma^H \le \tfrac{1}{2}\, \epsilon(1-\gamma)\,.
\end{align*}
Owning to that the immediate rewards belong to $[-1,1]$, it follows that for any policy $\pi$,
\begin{align}
q^\pi - {q}_{P'}^\pi \leq q^\pi_{H} - {q}^\pi_{P', H} + \varepsilon \mathbf{1} \,,
\label{eq:trunc}
\end{align}
where we use $q^{\pi}_{P', H}$ to denote the $H$-step value function under transition model $P'$.  
Define $N_h(s,a)$ as the number of episodes when the $h$th state-action pair in the episode is $(s,a)$. Note that $N(s,a)\ge N_h(s,a)$. 
Let $Z_h = \{ (s,a)\in \cS \times \cA \,: \, \nu^\pi_{\mu, h}(s,a)> \frac{8}{m}\ln \frac{\rS \rA}{\delta} \}$
and let $\cF$ be the event when 
\begin{align*}
\frac{N_h(s,a)}{m} \ge \frac{\nu_{\mu,h}^{\logpi}(s,a)}{2}
\end{align*}
holds for any $(s,a)\in Z_h$.%
\footnote{Note that $Z_h$, and thus also $\cF$ depends on $\pi$, which is the reason that the result, as stated, holds only for a fixed policy.}
By \cref{lem:ber}, $\sP(\cF)\ge 1-\delta$.

Assume that $\cF$ holds. 
Combining \cref{eq:trunc} with \cref{lem:H-step-error-decompose}, we get that on this event
\begin{align*}
\MoveEqLeft v^\pi(\mu) - \hat{v}^\pi(\mu) 
 \leq (\mu^\pi)^\top (q^\pi_{H} - {q}^\pi_{P', H} ) + \varepsilon \\
&= \gamma \sum_{h=0}^{H-1} \gamma^h ({\nu}^\pi_{\mu, h})^\top ({P} - {P'} ) {v}^\pi_{P', H-h-1} + \varepsilon
		 \tag{by \cref{eq:ed1}} \\
& \leq  \frac{\gamma}{1-\gamma} \sum_{h=0}^{H-1}  \gamma^h \sum_{s,a} {\nu}^\pi_{\mu, h}(s,a) \Vert {P}(\cdot| s,a) - {P'}(\cdot| s,a)  \Vert_{1} + \varepsilon 
		 \tag{by $\|\hat v_{H-h-1}^\pi\|_\infty \le 1/(1-\gamma)$} \\
& \le    \frac{\gamma}{1-\gamma} \sum_{h=0}^{H-1}  \gamma^h \sum_{(s,a)\in Z_h} {\nu}^\pi_{\mu, h}(s,a) \Vert {P}(\cdot| s,a) - {P'}(\cdot| s,a)  \Vert_{1}  + \frac{8\gamma \rS \rA}{m(1-\gamma)^2} \ln \frac{\rS \rA}{\delta}  + \varepsilon 
\tag{by the definition of $Z_h$}
\\
& \le    \frac{\gamma}{1-\gamma} \sum_{h=0}^{H-1}  \gamma^h \sum_{(s,a)\in Z_h} \sqrt{{\nu}^\pi_{\mu, h}(s,a)} \Vert {P}(\cdot| s,a) - {P'}(\cdot| s,a)  \Vert_{1}  + \frac{8\gamma \rS \rA}{m(1-\gamma)^2} \ln \frac{\rS \rA}{\delta}  + \varepsilon 
\tag{by ${\nu}^\pi_{\mu, h}(s,a)\leq 1$}
\\
&\le \frac{\gamma\rA^{\min(H, \rS) / 2} }{1-\gamma} \sum_{h=0}^{H-1}  \gamma^h   \sum_{(s,a)\in Z_h} \sqrt{{\nu}^\logpi_{\mu, h}(s,a)}  \Vert {P'}(\cdot| s,a) - {P}(\cdot| s,a) \Vert_1 + \frac{8\gamma \rS \rA}{m(1-\gamma)^2} \ln \frac{\rS \rA}{\delta}  + \varepsilon   
			\tag{by \cref{prop:ratbd} } \\
&\le \frac{2\gamma\rA^{\min(H, \rS) / 2}C}{1-\gamma} \sum_{h=0}^{H-1}   \gamma^h \sum_{(s,a)\in Z_h} \sqrt{{\nu}^\logpi_{\mu, h}(s,a)} \frac{1}{\sqrt{N_h(s,a)\vee 1}} + \frac{8\gamma \rS \rA}{m(1-\gamma)^2} \ln \frac{\rS \rA}{\delta}  + \varepsilon   
\tag{by the definition of $P'$}
\\
&\le \frac{2\gamma\rA^{\min(H, \rS) / 2}C}{1-\gamma} \sum_{h=0}^{H-1}   \gamma^h \sum_{(s,a)\in Z_h} \sqrt{{\nu}^\logpi_{\mu, h}(s,a)} \sqrt{\frac{2}{m {\nu}^\logpi_{\mu, h}(s,a)}} + \frac{8\gamma \rS \rA}{m(1-\gamma)^2} \ln \frac{\rS \rA}{\delta}  + \varepsilon   
\tag{by the definitions of $\cF$ and $Z_h$} 
\\
&\leq \frac{4\gamma \rA^{\min(H, \rS) / 2}C}{(1-\gamma)^{2}\sqrt{m}}  {\rS \rA} +  \frac{8\gamma \rS \rA}{m(1-\gamma)^2} \ln \frac{\rS \rA}{\delta} + \varepsilon \, .
\end{align*}
This finishes the proof. 
\end{proof}

For the plug-in method we use the previous lemma with $P'=\hat P$, resulting in the following corollary:

\begin{corollary}
\label{cor:fixed-pi-hatP-subopt-gap}
Let $\delta\in(0,1)$ and $m$ be the number of episodes collected by the logging policy and fix any policy $\pi$. 
With probability at least $1-\delta$,
 we have
\begin{align*}
v^\pi(\mu) - {v}_{\hat{P}}^\pi(\mu)
\le
   \frac{8\gamma \rS \rA^{\frac{\min(H, \rS)}{2}+1}}{(1-\gamma)^{2}}
 		\sqrt{\frac{ \rS\ln 2 + \ln\frac{2n(n+1)\rS\rA}\delta}{2m}} 
	 +  \frac{8\gamma \rS \rA}{(1-\gamma)^2} \frac{\ln \frac{2\rS \rA}{\delta} }{m}
	 + \varepsilon \, .
\end{align*}
\end{corollary}

\begin{proof}
Fix $\delta\in (0,1)$.
Let $\cE_\delta$ be the event when for any $(s,a)\in \cS \times \cA$,
\begin{align}
\norm{ \hat{P}(\cdot|s,a) - P(\cdot|s,a) }_1 \le
\beta(N(s,a),\delta)\,,\label{eq:pdefplugin}
\end{align}
where $\beta$ is defined in \cref{lem:empdev}, which also gives that
$\sP(\cE_\delta)\ge 1-\delta$. Further, 
defining
\begin{align*}
C_\delta = 2 \sqrt{\frac{\rS\ln 2 + \ln\frac{n(n+1)\rS\rA}\delta}{2}}\,,
\end{align*}
note that
$
\beta(u,\delta)\leq {C_\delta} / {\sqrt{u\vee 1}}
$. 
Now, let $\cF_\delta$ be the event when the conclusion of  \cref{lem:fixed-pi-subopt-gap} holds.
Then, on 
the one hand, by a union bound, $\PP(\cE_{\delta/2}\cap \cF_{\delta/2})\ge 1-\delta$, while on the other hand on 
$\cE_{\delta/2}\cap \cF_{\delta/2}$, the condition of 
\cref{lem:fixed-pi-subopt-gap} holds for 
$P'$ defined so that
\begin{align*}
P'(\cdot|s,a) = \begin{cases}
\hat P(\cdot|s,a) \,, & \text{if } \|\hat P(\cdot|s,a)-P(\cdot|s,a)\|_1 \le \beta(N(s,a),\delta/2)\,;\\
P(\cdot|s,a)\,, & \text{otherwise}\,.
\end{cases}
\end{align*}
with $C:=C_{\delta/2}$. 

Furthermore, on $\cE_{\delta/2}$, $\hat P(\cdot|s,a) = P(\cdot|s,a)$ holds for any $(s,a)$ pair.
Hence, the result follows by replacing $\delta$ with $\delta/2$ in \cref{eq:vdbd} and plugging in $C_{\delta/2}$ in place of $C$.
\end{proof}

We now are ready to prove the upper bound of plug-in algorithm. 

\begin{theorem}[Restatement of \cref{thm:plug-in-ub}]
\label{thm:plug-in-ub-new}
Fix $\rS$, $\rA$, an MDP $M\in \cM(\rS,\rA)$ and a distribution $\mu$ on the state space of $M$.
Suppose $\delta>0$, $\varepsilon>0$, and $\varepsilon_{\text{opt}}>0$. 
Assume that the data is collected by following the uniform mix of all deterministic policies and it consists of $m$ episodes, each of length $H=H_{\gamma, (1-\gamma)\varepsilon/(2\gamma)}$.
Let $\hat{\pi}$ be any deterministic, $\varepsilon_{\text{opt}}$-optimal policy for $\hat{M}=(\hat{P},r)$ where $\hat{P}$ is the sample-mean based empirical estimate of the transition probabilities based on the data collected.
Then if 
\begin{align*}
m= \tilde{\Omega} \left(  \frac{\rS^3 \rA^{\min(H, \rS) + 2}  \ln\frac{1}\delta }{(1-\gamma)^{4}\epsilon^2} \right),
\end{align*}
where $\tilde\Omega$ hides log factors of $\rS,\rA$ and $H$,
we have 
$v^{\hat{\pi}}(\mu) \ge v^*(\mu) -4\varepsilon - \varepsilon_{\text{opt}}$
with probability at least $1-\delta$. 
\end{theorem}
\begin{proof}
We upper bound the suboptimality gap of $\hat{\pi}$ as follows: 
\begin{align*}
\MoveEqLeft v^*(\mu) - v^{\hat{\pi}}(\mu)  
= v^*(\mu) - \hat{v}^{\pi^*}(\mu) + \hat{v}^{\pi^*}(\mu)
	- \hat{v}^{\hat\pi}(\mu) + \hat{v}^{\hat\pi}(\mu) - v^{\hat\pi} (\mu)\\
& \le 
v^*(\mu) - \hat{v}^{\pi^*}(\mu) + \hat{v}^{\hat\pi} (\mu)- v^{\hat\pi}(\mu) + \varepsilon_{\text{opt}}\,.
			\tag{ $\hat{\pi}$ is $\varepsilon_{\text{opt}}$-optimal  in $\hat{M}$} 
\end{align*}
By \cref{cor:fixed-pi-hatP-subopt-gap} and a union bound, 
with probability at least $1-\delta$, for any deterministic policy $\pi$ obtained from the data $\cD$ we have
\begin{align*}
\MoveEqLeft v^{\pi}(\mu) - \hat{v}^{\pi}(\mu) \\
& \le
     \frac{8\gamma \rS \rA^{\frac{\min(H, \rS)}{2}+1}}{(1-\gamma)^{2}}
 		\sqrt{\frac{ \rS\ln 2 + \ln\frac{2n(n+1)\rS\rA}\delta + \rS\ln\rA}{2m}} 
	 +  \frac{8\gamma \rS \rA}{(1-\gamma)^2} \frac{\ln \frac{2\rS \rA}{\delta} + \rS \ln \rA }{m}
	 + \varepsilon 
\\
& \leq  \frac{8\gamma \rS^{\frac{3}{2}} \rA^{\frac{\min(H, \rS)}{2} + 1}}{(1-\gamma)^{2}}
 		\sqrt{\frac{ \ln 2 + \ln\frac{H^2 \rS\rA}\delta + \ln 2m + \ln\rA }{2m}} 
	 +  \frac{8\gamma \rS \rA}{(1-\gamma)^2} \frac{\ln \frac{2\rS \rA}{\delta} + \rS \ln \rA }{m}
	 + \varepsilon \, ,
\end{align*}
Thus, given that
\begin{align*}
m  = \tilde{\Omega} \left(  \frac{\rS^3 \rA^{\min(H, \rS) + 2}  \ln\frac{1}\delta }{(1-\gamma)^{4}\epsilon^2} \right)\, ,
\end{align*}
where $\tilde \Omega$ hides log-factors,
with probability at least $1-\delta$ we have,
\begin{align*}
v^*(\mu) - v^{\hat{\pi}}(\mu)   \leq  v^*(\mu) - \hat{v}^{\pi^*}(\mu) + \hat{v}^{\hat{\pi}^*}(\mu) - v^{\hat{\pi}^*}(\mu) + \varepsilon_{\text{opt}} \leq 4\varepsilon + \varepsilon_{\text{opt}}\, .
\end{align*}
\end{proof}


\subsection{Pessimistic Algorithm}
\label{sec:pess}
\newcommand{\eopt}{\epsilon_{\mathrm{opt}}}
We present a result in this section for the ``pessimistic algorithm'' in the discounted total expected reward criterion to complement the results in the main text \citep{jin2021pessimism,BuGeBe20,KiRaNeJo20,yu2020mopo,kumar2020conservative,liu2020provably,Yu2021combo}.
The sample complexity we show is the same as for the plug-in method. While this may be off by a polynomial factor, we do not expect the pessimistic algorithm to have a major advantage over the plug-in method in the worst-case setting.
In fact, the recent work of \citet{xiao2021optimality} established this in a rigorous fashion for the bandit setting by showing an algorithm independent lower bound that matched the upper bound for both the plug-in method and the pessimistic algorithm. 
As argued by \citet{xiao2021optimality} (and proved by \citet{jin2021pessimism} in the context of linear MDPs, which includes tabular MDPs), the advantage of the pessimistic algorithm is that it is weighted minimax optimal with respect to a special criterion.

The pessimistic algorithm with parameters $\delta\in (0,1)$ and $\eopt>0$ 
chooses a deterministic $\eopt$ policy $\tilde \pi$ of the MDP with reward $r$ and transition kernel $\tilde P$, the latter of which is obtained via
\begin{align*}
\tilde P = \argmin_{P'\in \cP_\delta} v_{P'}^*(\mu)\,,
\end{align*}
where for a transition kernel $P'$ we use $v_{P'}^*$ to denote the optimal value function in the MDP with immediate rewards $r$ and transition kernel $P'$,
and $\cP_\delta$ is is defined as
\begin{align*}
\cP_\delta = \left\{ P' \,:\, \text{for any } (s,a)\in \cS \times \cA\,,
\norm{ \hat{P}(\cdot|s,a) - P'(\cdot|s,a) }_1 \le
\beta(N(s,a),\delta)\, \right\}\,,
\end{align*}
where $\beta$ is defined in \cref{lem:empdev}.
Recall that the same result ensures that $P$, the ``true'' transition kernel belongs to $\cP_\delta$ with probability at least $1-\delta$.

\begin{theorem}[Pessimistic algorithm]
\label{thm:pess-ub-new}
Fix $\rS$, $\rA$, an MDP $M\in \cM(\rS,\rA)$ and a distribution $\mu$ on the state space of $M$.
Suppose $\delta>0$, $\varepsilon>0$, and $\eopt>0$. 
Assume that the data is collected by following the uniform mix of all deterministic policies and it consists of $m$ episodes, each of length $H=H_{\gamma, (1-\gamma)\varepsilon/(2\gamma)}$.
Then, if 
\begin{align*}
m= \tilde{\Omega} \left(  \frac{\rS^3 \rA^{\min(H, \rS) + 2}  \ln\frac{1}\delta }{(1-\gamma)^{4}\epsilon^2} \right),
\end{align*}
where $\tilde\Omega$ hides log factors of $\rS,\rA$ and $H$,
we have 
$v^{\tilde{\pi}}(\mu) \ge v^*(\mu) -2\varepsilon - \varepsilon_{\text{opt}}$
with probability at least $1-\delta$, where $\tilde \pi$ is the output of the pessimistic algorithm run with parameters $(\delta,\eopt)$. 
\end{theorem}
\begin{proof}
Let us denote by $v_{P'}^\pi$ the value function of policy $\pi$ in the MDP with reward $r$ and transition kernel $P'$.
Let $\Delta^\pi = v_{\tilde{P}}^{\pi}(\mu)-{v}^{\pi}(\mu)$ and let $\pi^*$ is an deterministic 
optimal policy in the ``true'' MDP. Such a policy exists (e.g., see Theorem 6.2.10 of \cite{puterman2014markov}). 
We have
\begin{align*}
v^*(\mu) - v^{\tilde{\pi}}(\mu) 
& =   v^*(\mu)-v_{\tilde{P}}^{\tilde\pi}(\mu) + \Delta^{\tilde\pi} \\
& \le v^*(\mu)-v_{\tilde{P}}^{*}(\mu) + \Delta^{\tilde\pi} +\eopt
	\tag{by the definition of $\tilde\pi$}
	\\
& \le v^*(\mu)-v_{\tilde{P}}^{\pi^*}(\mu) + \Delta^{\tilde\pi} 	+\eopt
	\tag{because $v_{\tilde{P}}^{\pi^*}\le v_{\tilde{P}}^*$}
	\\
& \le \Delta^{\pi^*} + \Delta^{\tilde \pi}+\eopt\,.
	\tag{because $v^*(\mu) = v^{\pi^*}(\mu)$}
\end{align*}

Hence, it remains to upper bound $\Delta^{\pi^*}$ and $\Delta^{\tilde \pi}$.
For this, we make the following claim:

\noindent \underline{Claim:} 
Fix any policy $\pi$. Then, with probability at least $1-\delta$, we have
\begin{align*}
v^\pi(\mu) - {v}_{\tilde{P}}^\pi(\mu)
\le
   \frac{16\gamma\rS \rA^{\frac{\min(H, \rS)}{2}+1}}{(1-\gamma)^{2}}
 		\sqrt{\frac{ \rS\ln 2 + \ln\frac{2n(n+1)\rS\rA}\delta}{2m}} 
	 +  \frac{8\gamma \rS \rA}{(1-\gamma)^2} \frac{\ln \frac{2\rS \rA}{\delta} }{m}
	 + \varepsilon \, .
\end{align*}
\begin{proof}[Proof of Claim]
For the latter, note that the proof of 
\cref{cor:fixed-pi-hatP-subopt-gap} can be repeated with the only change that now instead of
\cref{eq:pdefplugin}, we have
\begin{align*}
\norm{ {P}(\cdot|s,a) - \tilde{P}(\cdot|s,a) }_1 
\le
\norm{ {P}(\cdot|s,a) - \hat{P}(\cdot|s,a) }_1 
\norm{ \hat{P}(\cdot|s,a) - \tilde{P}(\cdot|s,a) }_1 
\le
2\beta(N(s,a),\delta)\, .
\end{align*}
\end{proof}
From this claim,
by a union bound over all the $\rA^{\rS}$ deterministic policies, we get that
with probability $1-\delta$, for any deterministic policy $\pi$,
\begin{align*}
\MoveEqLeft v^\pi(\mu) - {v}_{\tilde{P}}^\pi(\mu) \\
& \le
   \frac{16\gamma \rS \rA^{\frac{\min(H, \rS)}{2}+1}}{(1-\gamma)^{2}}
 		\sqrt{\frac{ \left(\rS\ln 2 + \ln\frac{2n(n+1)\rS\rA}\delta + \rS\ln\rA\right)}{2m}} 
	 +  \frac{8\gamma \rS \rA}{(1-\gamma)^2} \frac{\ln \frac{2\rS \rA}{\delta}+ \rS\ln\rA}{m}
	 + \varepsilon \, .
\end{align*}
Since $\tilde\pi$, by definition is also a deterministic policy, the last display holds with probability $1-\delta$ for $\tilde \pi$ as well.
Putting things together gives that
\begin{align*}
\MoveEqLeft 
v^*(\mu) - v^{\tilde\pi}(\mu) \\
& \le
   \frac{32\gamma \rS \rA^{\frac{\min(H, \rS)}{2}+1}}{(1-\gamma)^{2}}
 		\sqrt{\frac{ \rS\ln 2 + \ln\frac{2n(n+1)\rS\rA}\delta + \rS\ln\rA}{2m}} 
	 +  \frac{16\gamma \rS \rA}{(1-\gamma)^2} \frac{\ln \frac{2\rS \rA}{\delta}+ \rS\ln\rA}{m}
	 + 2\varepsilon  +\eopt\, .
\end{align*}
The proof is finished by a calculation similar to that done at the end of the proof of \cref{thm:plug-in-ub-new}.
\end{proof}

\end{document}